\documentclass[final,onefignum,oneeqnum,onethmnum]{siamonline171218}

\usepackage{graphicx}
\usepackage{amsmath}
\usepackage{enumitem}
\usepackage{bm}
\usepackage{mathtools}
\usepackage{calc}
\usepackage{amssymb}

\newtheorem{innercustomthm}{Lemma}
\newenvironment{customthm}[1]
  {\renewcommand\theinnercustomthm{#1}\innercustomthm}
  {\endinnercustomthm}

\title{Critical Contours: An Invariant Linking Image Flow\\ with Salient Surface Organization\thanks{Received by the editors
September 1, 2017; accepted for publication (in revised form) March 26, 2018; published electronically July 12, 2018.
\URL{siims/11-3/M114552.html}\funding{The research of the authors was supported by the NSF, the Paul G. Allen Family Foundation, and the Simons Foundation.}}}

\author{Benjamin Kunsberg\thanks{Division of Applied Mathematics, Brown University, Providence, RI 02906 (\href{mailto:benjamin_kunsberg@brown.edu} {benjamin\_kunsberg@brown.edu}).}
\and Steven W. Zucker\thanks{Department of Computer Science, Yale University, New Haven, CT 06511 (\email{zucker@cs.yale.edu}).}}

\headers{Critical Contours}{Benjamin Kunsberg and Steven W. Zucker}

\begin{document}

\slugger{siims}{2018}{11}{3}{1849--1877}
\maketitle

\setcounter{page}{1849}

\begin{abstract}
We exploit a key result from visual psychophysics---that individuals perceive shape qualitatively---to develop the use of a
geometrical/topological ``invariant'' (the Morse--Smale complex) relating image structure with surface structure. Differences
across individuals are minimal near certain configurations such as ridges and boundaries, and it is these configurations that
are often represented in line drawings. In particular, we introduce a method for inferring a qualitative three-dimensional shape from shading
patterns that link the shape-from-shading inference with shape-from-contour inference. For a given shape, certain shading
patches approach ``line drawings'' in a well-defined limit. Under this limit, and invariably with respect to rendering choices,
these shading patterns provide a qualitative description of the surface. We further show that, under this model, the contours
partition the surface into meaningful parts using the Morse--Smale complex. These critical contours are the (perceptually) stable
parts of this complex and are invariant over a wide class of rendering models. Intuitively, our main result shows that critical
contours partition smooth surfaces into bumps and valleys, in effect providing a scaffold on the image from which a full surface
can be interpolated.
\end{abstract}



\begin{DOI}
10.1137/17M1145525
\end{DOI}

\section{Introduction}

Mathematically it is well known that the problem of inferring shape from shading information---or from contour information---is ill-posed: there exist many possible surfaces that could give rise to (almost) any type of image structure.  In our everyday experience, however, we unconsciously ``solve'' this inverse problem routinely; we readily and effortlessly infer three-dimensional (3D) shape from ambiguous image information.  This presents a huge conundrum for theorists: does there exist an invariant that could ground inferences about surface structure on the many different types of image structures and, if so, how might this invariant be used by both brains and machines? In this paper we answer the first part of this question in the affirmative, motivated by two aspects of the second part.

Our approach is motivated by an important (but frequently overlooked) property of human perception:  different individuals (or the same individual at different times) perceive quantitatively different but qualitatively similar surfaces---not identical ones---from either shading or contour information (references in the background section below). We take this property to be key: the goal is not to find a unique map between an image and a surface, but rather to identify an equivalence class structure, i.e., to identify which classes of images are consistent with which classes of surfaces. The identification map is then at this abstract level. As we shall show, there are parts of images that do indeed provide a kind of scaffold from which (parts of) surfaces can be reconstructed. We call this scaffold {\it critical contours}; in this paper we define them, prove that they are part of a global description, and characterize their invariance over different rendering models. Basically, we show that critical contours in the image are nearly equivalent to critical contours of the surface (slant) function.

We concentrate on image ridges, a geometrical construct that has been important in vision for some time. Viewing the image as a height function,  ridges seem intuitively connected to image edges \cite{haralick1983ridges, eberly1994ridges}, especially those that arise within the interior of a shape \cite{Hallinan:1999:TTP:307403}. For similar reasons, ridges also relate to features of
surface relief \cite{koenderink1994two}. But, to our knowledge, image ridges and surface relief have not been formally identified with one another, except for specific rendering models (e.g., Lambertian). We establish this connection generically; see Figure~\ref{fig:blob}.

\begin{figure}[t]
\begin{center}
\includegraphics[width = 1 \linewidth]{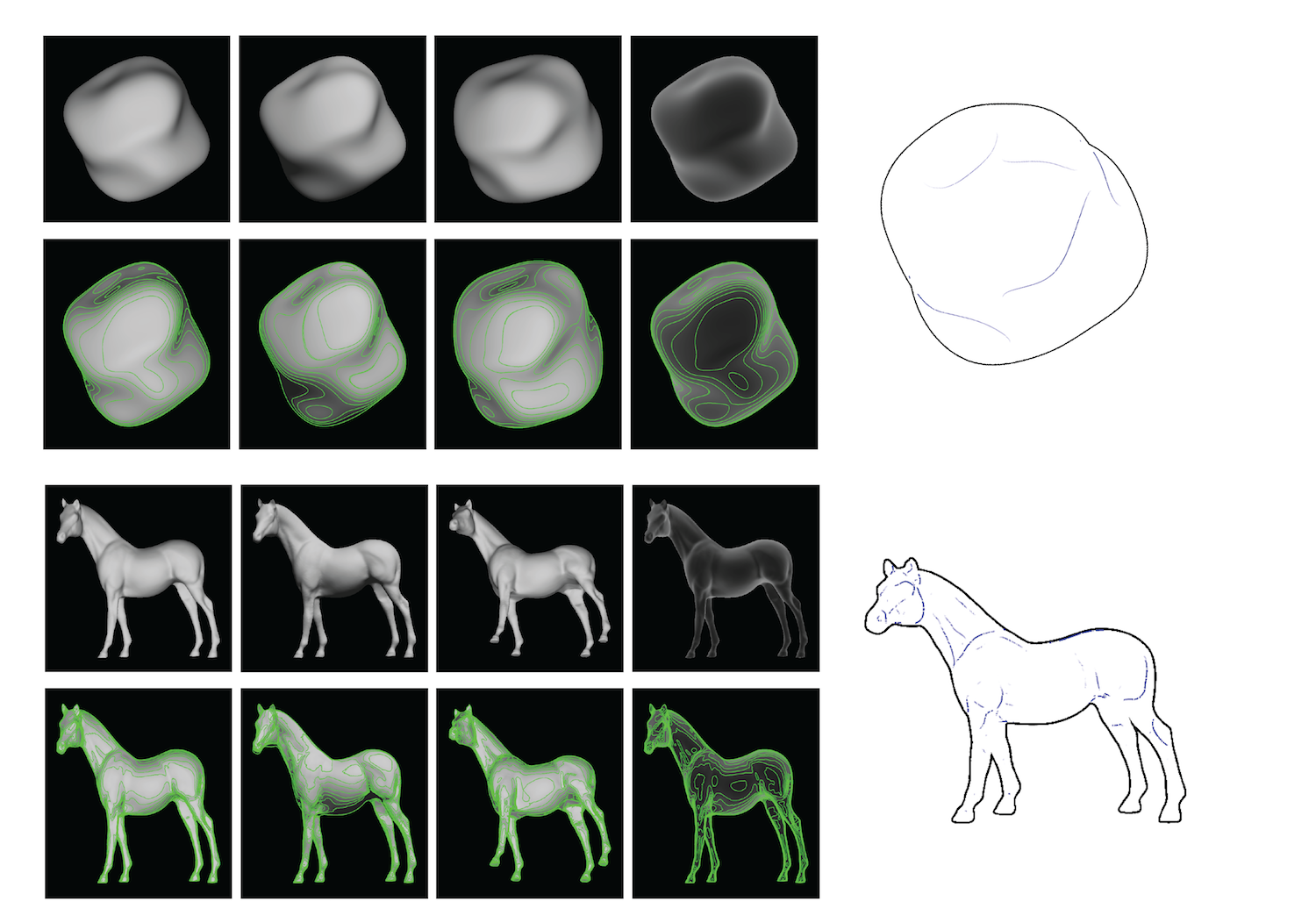}
\caption{Qualitative descriptions of $3$D shapes from unknown cues are stable.
Row {\rm 1:} A random shape with different lightings, a slight rotation, and a nonmonotonic transformation ($\arctan(x)$) as in {\rm \cite{Fleming14}}.
Row {\rm 2:} Isophote patterns generally vary for each corresponding case above yet are stable in some positions (e.g., along the sides of the protrusion).
Rows {\rm 3, 4:} Similar to Rows $1$, $2$ but with a horse model.
Right: Suggestive contours {\rm \cite{DeCarlo03}} for the same models. Note the similarity between the stable flow patterns and the suggestive contour positions.
The method we introduce is inspired by such line drawing images. We will develop this visual commonality (critical contours) using vector field topology.\vspace*{-2pt}
\label{fig:blob}}
\end{center}
\end{figure}

The development proceeds in three main steps.
We first exploit ridge structure to motivate a limiting process that characterizes how  shading distributions concentrate into contours. Second, when the shading distribution orthogonal to the contour is sufficiently ``steep,'' it becomes what we define as a {\em critical contour}. These are special contours that capture the edge connection alluded to above and resemble artists' line drawings. Building on the surface relief view, they form part of a global, topological network that separates ``hills'' from ``dales.''  Formally this network comprises the Morse--Smale (MS) complex on the image \cite{Gyulassy08} and is built with integral curves through the (image) gradient flow that connect maxima, saddles, and minima in a prescribed manner.
The MS idea has a rich history in geology \cite{maxwell1870hills, cayley1859xl} but in the modern form is based on singularities of gradient flows \cite{smale1961gradient, milnor2016morse}. This has three important consequences:
(i) it allows a principled (global) simplification process to remove insignificant ``bumps'' \cite{edelsbrunner2010computational}; (ii) the flows ground the computations in
physiologically meaningful terms \cite{Hubel88}; and (iii) it allows the contours to be interpreted as boundaries of surface parts.

Finally, we show that critical contours are part of the MS complex in a generic sense.
Since the natural world is hardly Lambertian, we consider a general class of rendering functions that require little more than dependence on the surface normal (or tangent plane at that point). We provide a completeness theorem: if there is a critical contour in the image for a given surface with one rendering function in the class, then there is a critical contour in the image for every rendering function in this class and, furthermore, they coincide under the limiting process mentioned above.
This third result has an unexpected implication: Since the surface slant function \cite{Stevens83}
is in our class of rendering functions, critical contours from the image and critical contours from the surface slant are, in the limit above, equivalent. Thus we relate image-derived properties directly to surface properties.

Since the MS complex is global, there is a shared partition between the image and the surface.
By identifying certain contour inferences with shading inferences, it also constrains how the surface can be ``filled in'' between the image contours. But it does not reduce the reconstructed surface to a singleton. In effect, by working between the geometry of ridges and the topology of surfaces, we are able to find a foundation for qualitative surface inferences. A roadmap for our approach is shown later in
Figure~\ref{fig:overview}; it is described in more detail following the background review.

\section{Background}

Standard computational approaches to shape-from-shading rely on either imposing strong priors (on light sources \cite{sun1998sun}, reflectance models \cite{langer1994shape}, etc.) or imposing a form of regularization tied to a reflectance model (reviews in \cite{horn89,horn85, faugeras, zheng91}; see Mach \cite{mach} for the original); in either case the goal is a single, unique surface from among the different possibilities \cite{Oliensis1991, Saxberg92}. This remains problematic:  Recent attempts are brittle for related reasons---they work for some images but not others, largely because of reliance on artificial reflectance models, bas relief priors \cite{Bel99},  a delicate combination of regularization terms \cite{Barron:2012tt}, or training on restricted scene classes \cite{2012arXiv1206.6445T, DBLP:journals/corr/EigenF14,
DBLP:journals/corr/ChakrabartiSS16}. We move beyond this brittleness by seeking a qualitative solution.

\subsection{Surface perception is qualitative}

Results in visual psychophysics question the goal of seeking a unique surface.  While subjects tend to agree on the overall shape,{\vadjust{\pagebreak}} constancy is elusive and percepts differ quantitatively; see \cite{doi:10.1167/12.1.12, Mooney20142737, MAMASSIAN19962351,Mingolla1986, doi:10.1068/i0645, doi:10.1167/15.2.24, CHRISTOU19971441, doi:10.1167/15.2.24, Seyama19983805,Curran19961399, doi:10.1068/p5807, doi:10.1068/p251009}. Remarkably, this lack of constancy holds even for special shapes such as cylinders (but see \cite{tensors-arxiv-1,tensors-arxiv-2}).
One possibility is that there are different operational modes \cite{doi:10.1167/12.1.12}; another is that priors are applied selectively to advantage \cite{doi:10.1167/13.5.10}. We propose that the solution is topological in nature.

The idea of different modes is consistent with the view, prominent in computer vision, that shape-from-contour is a separate problem from shape-from-shading: Contours are one-dimensional entities, while shading is a two-dimensional distribution of intensities. For contours the emphasis tends to be on junctions \cite{Waltz75understandingline}---the places where surfaces join---to make sure that the surfaces ``fit'' together properly; see \cite{Barrow:1981:ILD:3015381.3015385, Stevens82, Malik1987, Huffman:1976:CCP:1311922.1312015}. Again, research in visual psychophysics provides a challenge to this separate view; mutual influences between contours and shading are well documented \cite{Ramachandran:1988wo,todorovic2014shape}. We shall show that well-defined, salient, and stable contours can arise out of shaded images; thus both contours and shading play the role of defining surface parts.  In our view, shape-from-shading and shape-from-contour are deeply related inverse problems; another of our contributions is to show how via a limiting process.


While the inference problem works on the inverse direction,
the relationship between shading and contour in the forward direction is well established.
Nonphotorealistic rendering algorithms are used as visualization techniques for given surfaces \cite{Lawonn2016, peikert1999parallel}, which shows how rich surface information can be conveyed to a viewer; this is not unlike what artists draw \cite{Cole09}.
For example, ``suggestive contours'' \cite{DeCarlo03} are built from the loci of points (in the image) where the object almost occludes itself {\it when computed from the surface} \cite{DeCarlo03}; see also \cite{Judd07, Sahner08} for related forward computations. Our critical contours---which work for the inverse problem---are related to, but not identical with, suggestive contours.  More details on the relationship can be found in the appendix.

Folds in material are a prominent example of when suggestive contours are useful (see da Vinci's notebooks and \cite{Lee:2007:LDV:1276377.1276400, Gingold:2008:SSE:1360612.1360694,Jung:2015:SFD:2843519.2749458}); we exploit the fact that folds have rather structured shading across them \cite{kunsberg2014shading}.
They tend to occur along extended anisotropic curvature regions of the shape \cite{Lawlor200918} and
are related to ridges in computer vision. Detecting these areas of rapid change in image intensity is a well-studied problem \cite{haralick1983ridges, lopez1999evaluation}, although
the characterization often remains local in terms of differential geometry
\cite{Hallinan:1999:TTP:307403} or singularity theory \cite{damon2016local}. But this local characterization may not yield globally connected patterns, and ``small'' ridges are no different from ``steep'' ones. Many have explored a multiscale approach to deal with these difficulties
\cite{lindeberg1998edge, pizer1998zoom, griffin1995superficial, eberly1994ridges}; our work exploits
a topological multiscale idea.

The transition to global patterns from local descriptors is necessary, and the classical work on ``hills and dales'' \cite{maxwell1870hills, cayley1859xl, rothe-talweg} provides a way forward. It considers the integration of vector fields, and
applications to describing waterways \cite{moore1991terrain, barnes2014priority} and images \cite{lonquet1960reflection, koenderink1998structure, griffin1995superficial, beucher1992morphological, lopez1999evaluation,
koenderink1994two} abound. Putting these together, the idea is to view the image as a ``landscape,'' with ``height'' proportional to intensity; ``water'' then flows from the peaks to the valleys. But viewing the image as a landscape does not provide
a formal connection back to the underlying surface from which it was rendered. Our critical contours, also computed from the image, will relate directly to the image landscapes sought by these ridge detectors but will further have a connection to the underlying surface.

It is now that we appeal directly to the psychophysical observation above; that perception is qualitatively similar but not quantitatively identical across subjects \cite{doi:10.1167/12.1.12, Mooney20142737, MAMASSIAN19962351,Mingolla1986, doi:10.1068/p5179, doi:10.1068/i0645, doi:10.1167/15.2.24, CHRISTOU19971441, doi:10.1167/15.2.24, Seyama19983805,Curran19961399, doi:10.1068/p5807, doi:10.1068/p251009, doi:10.1167/12.1.2}, plus many others.
To us ``qualitative'' implies that we should be seeking that family of surfaces which are ``locked down'' by the image. Since global, qualitative constructions are the domain of topology, it is here that the MS complex \cite{Gyulassy08, edelsbrunner2010computational}
is central. We review the MS complex in the next section; later we shall show that the critical contours are 1-cells of the MS complex of the shading function with high transversal intensity changes.
It follows, then, that they are also 1-cells of the slant (foreshortening) function on the surface.

\subsection{The Morse--Smale complex}

Our goal is to find patterns in the image, computable from orientations and invariant to a large class of rendering functions, that ``anchor'' the ill-posed shape-from-shading problem in a qualitative manner.   We were inspired by representations of the phase space in nonlinear dynamics and wish to understand how particular contours on the image can constrain global, qualitative shape. It formalizes an intuition from Koenderink (and Picasso): that ``vision grasps shape as a hierarchical structure of elliptic patches'' \cite{Koenderink82Perception}.
For this, we use the MS complex. Like the watershed algorithms referenced above, the gradient flow is used to assign different regions of the domain of critical points; 2D contours separate these domains into monotonic regions (called 2-cells); these regions are then the ``parts'' of the shape.
Importantly, associated with the MS complex is
persistence simplification, a principled way to collapse critical points (equivalently, merge the watershed regions) to create a hierarchy \cite{Gyulassy08}.  This is the multiscale component of our approach.
This introduction is necessarily brief. More complete treatments can be found in \cite{milnor2016morse, forman1998morse, forman2002user, matsumoto2002introduction, Biasotti:2008:DSG:1391729.1391731} and, for motivation, see \cite{milnor1997topology}. See Figure \ref{Fig:MS}  for an illustration.

Given a $2$-manifold $\mathbb{M}$, consider a smooth scalar function $f(x, y): \mathbb{M} \rightarrow \mathbb{R}$.  (Later, we will consider $M = \mathbb{R}^2$ and $f$ as an image of a surface.)  The \emph{gradient}
$\nabla f = \left( \partial f / \partial x, \partial f / \partial y \right)$
exists at every point.  A point $p \in \mathbb{R}^2$ is called a \emph{critical point} when $\nabla f (p) = 0$.  The function $f$ is a \emph{Morse function} if all its critical points are nondegenerate (meaning the Hessian at those points is nonsingular) and if no two critical points have the same function value.

The gradient field gives a direction at every point in the image, except for the critical points, a set of measure zero. Following the vector field will trace out an \emph{integral line}.  Precisely, an integral line is a maximal path on the image whose tangent vectors agree with $\nabla f$ at every point of the path.  These integral lines must end at critical points, where the gradient direction is undefined.  Thus, one can define an \emph{origin} and a \emph{destination} for each integral line.  Further, for each critical point,  its \emph{ascending manifold} is defined as the union of integral lines having that critical point as a common origin.  Similarly, its \emph{descending manifold} is the union of integral lines with that critical point as a common destination.

The type of each critical point is defined by its \emph{index}: the number of negative eigenvalues of the Hessian at that point.  For scalar functions on $\mathbb{R}^2$, there are only three types: a maximum (with index 2), a minimum (with index 0), and a saddle point (with index 1).  MS functions satisfy an additional transversality condition and are dense in the set of continuous functions.  For these, the integral lines only connect critical points of differing index.  The ascending manifold associated with a critical point of index $k$ is of dimension $2 -k$.  Similarly, the descending manifold for an index $k$ critical point is dimension $k$.

For two critical points $p$ and $q$, with the index of $p$ one greater than the index of $q$,  consider the intersection of the descending manifold of $p$ with the ascending manifold of $q$.  This intersection will be either a 1D manifold (a curve called a 1-cell or watershed) or the empty set.   For two critical points $r$ and $s$, with the index of $r$ two greater than the index of $s$, the intersection of the descending manifold of $r$ with the ascending manifold of $s$ will either be a 2D manifold (a region called a 2-cell) or the empty set.  Thus, the intersection of all ascending manifolds with all descending manifolds partitions the manifold $\mathbb{M}$ into 2D regions surrounded by 1D curves with intersections at the critical points.

The MS complex is a structure that relates a set of contours to a qualitative function representation.  With knowledge only of the scalar function at the critical points and 1-cells, one could reconstruct the 2-cells (and thus the entire function) relatively accurately.  For some insight, see \cite{Giorgis15, Weinkauf10}.  The position, heights, and boundaries of all the bumps, dimples, and ridges are already known and the choices left are how steep to make the transitioning 2-cells in between.  Thus, there is a natural connection between the scalar function restricted to the 2D curves (the salient 1-cells) and the scalar function on the entire domain (the unknown 2-cells).
Our main theorem will show that particular 1-cells of the image will be nearly invariant under changes in the rendering function.

\begin{figure}[!b]\vspace*{-6pt}
\begin{center}
 \includegraphics[width = 0.7 \linewidth]{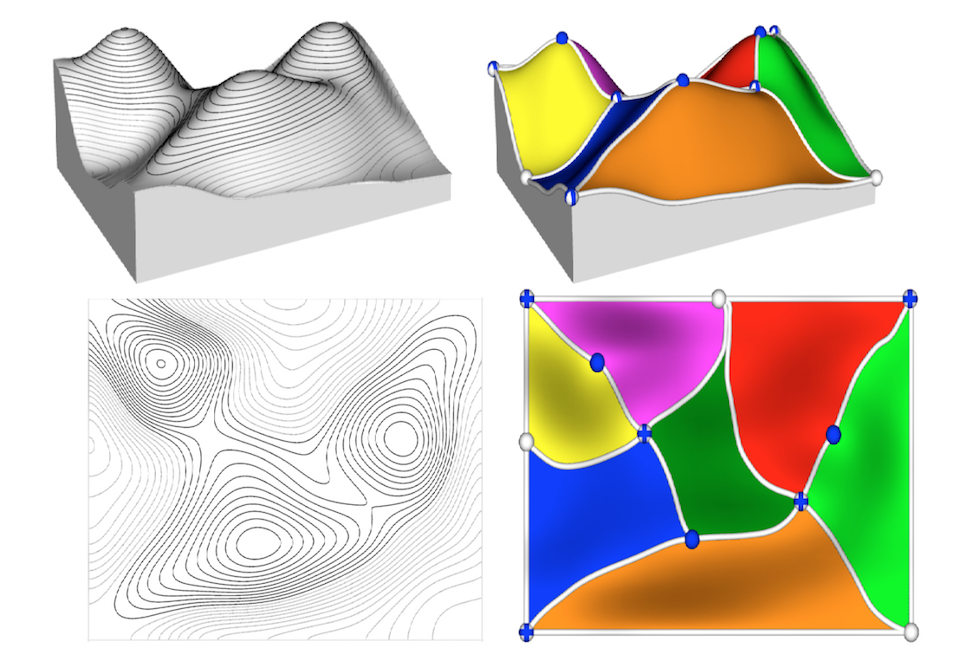}
\caption{Illustration of the MS complex for a scalar function in two dimensions.
(left column) A ``mountain range'' seen from the side and from above. Contours are level sets in height. If this scalar function was image intensity, the level curves would be isophotes. Colored regions represent $2$-cells of the MS complex.  White curves represent $1$-cells (contours) of the MS complex.  Maxima, saddles, and minima are represented by solid blue points, crosses, and solid white points, respectively. Each $2$-cell is combinatorially a quadrilateral, but it is possible that the saddles may be identified together creating a loop. This figure illustrates how a set of $2$D contours can represent the $3$D surface, up to monotonic transformations on each region. Figure from {\rm \cite{Gyulassy08}}.  \label{Fig:MS}}
\end{center}
\end{figure}

\subsection{Biological considerations}

Since our approach is partly motivated by psychophysical considerations, we also consider the underlying physiology; Connor \cite[Figure~2]{Connor08}, shows the sensitivity of higher-level neurons to ridge-like structure. Our concern here is how this process can get started.
Orientation-selective neurons in the visual cortex are the natural substrate for representing shading information as a flow pattern \cite{Koenderink:1980bm, zucker-sff, Zucker04}, and it clearly relates to the gradient flow above. Shape-from-shading-flow computations have been analyzed \cite{Kunsberg2014}, and research supports it \cite{fleming, doi:10.1167/14.7.1, DanVSS13, doi:10.1167/9.11.10, doi:10.1167/14.7.1}.
But Todd \cite{doi:10.1068/i0645, doi:10.1167/15.2.24}, among others, has questioned (forcefully and, to us, in an influential way) whether isophotes and shading flows suffice, because the isophote pattern changes significantly for different renderings and lightings of the same object (cf. Figure~\ref{fig:blob}, row 2), but our perception hardly varies (with regard to shape).  If our percepts were based on the isophotes alone, then they, too, Todd argues, should also change. But isophotes change more in some places than others, and the conditions in our shading limit proposition identify precisely those locations where the isophote structure remains invariant.  Anchoring the shape reconstruction on the locations where the isophote structure is consistent could explain how it is possible for brains to make robust (but qualitative) inferences about shape in three dimensions. Neural responses should be robust around critical contours, but not necessarily elsewhere, which implies that different positions are represented differently. (Earlier computational approaches treat all positions as equals.)

\subsection{Overview}

\begin{figure}[h!]
\begin{center}
\includegraphics[width = 1 \linewidth]{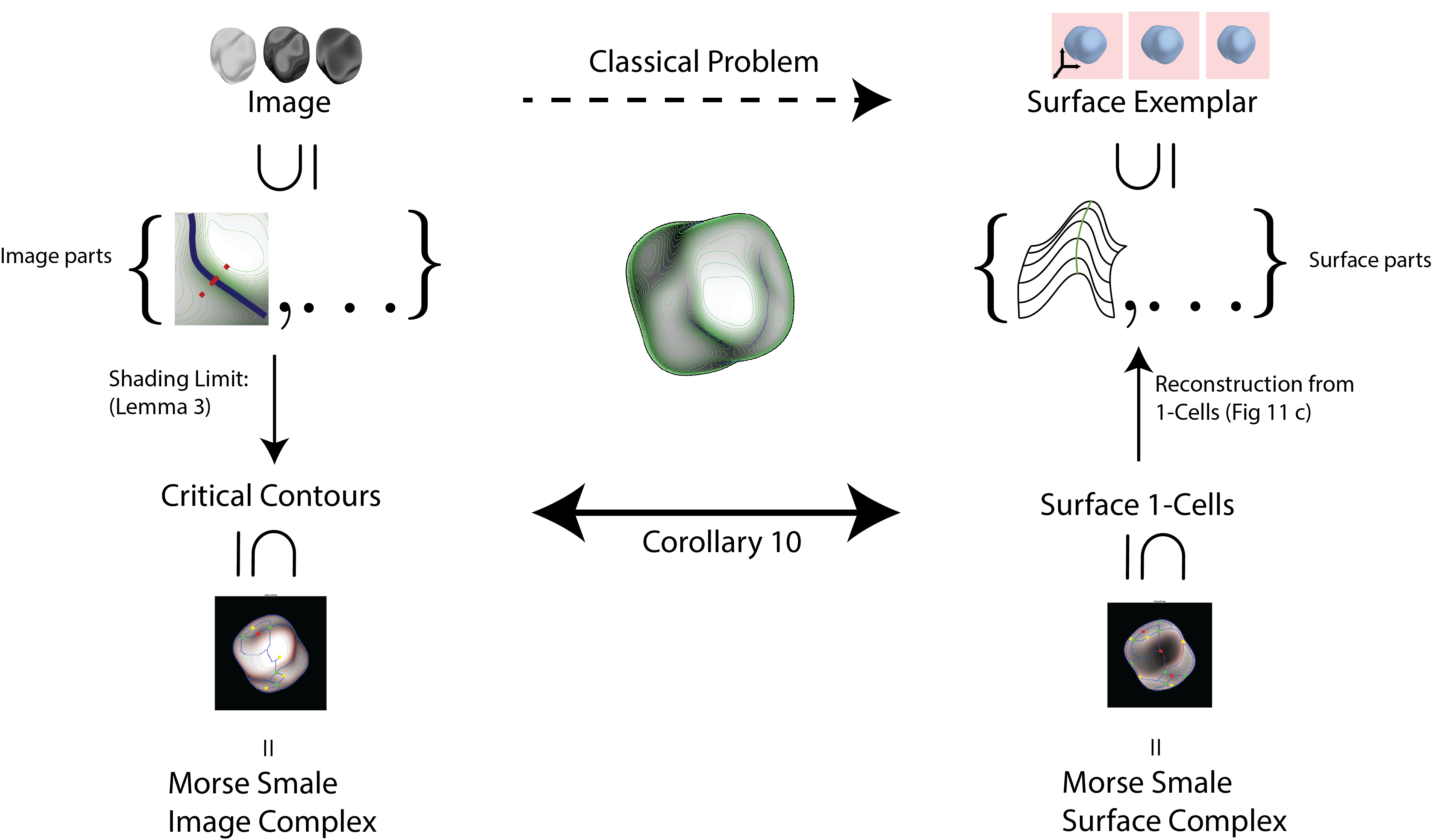}
\caption{Overview of our approach, starting in the upper left corner. Suppose we are given an image of a surface created via an unknown rendering function.  Classical shape-from-shading methods attempt to infer, from this pixel representation, a unique surface. We follow the vertical path  and identify those image features that will correspond to surface features. These are the critical contours that delimit ``image parts.'' Critical contours are invariant to the (unknown) rendering function. We show specifically how shading ``concentrates'' into such contours; we call this the shading-contour limit.  This allows us to move to ``critical contours,'' defined in section $7$, and  Corollary {\rm \ref{cor:slant}} allows us to interpret these critical contours as surface curves with important properties. We arrive at ``surface $1$-cells,'' which correspond to qualitative parts of the surface (bumps, valleys, ridges, etc.); these can be completed back into a scalar field such as surface slant, which leads, finally, to a surface exemplar in the upper right.
}
\label{fig:overview}
\end{center}
\end{figure}

An overview of our argument is as follows (Figure~\ref{fig:overview}). Starting in the upper left corner, we are given an image of a surface created via an unknown rendering function. It might be a shaded image,  a line drawing, etc.  Classical shape-from-shading methods compute directly from this pixel representation and must confront this ill-posed problem (dotted right arrow).  Instead, we proceed to first identify image features (that will correspond to surface features) that are invariant to the (unknown) rendering function; this is shown as ``image parts.''  For shading functions, these image parts are abstracted to stylized lines in section 5; the relationship is summarized in Lemma \ref{lemma:shading_derivs}.   This allows us to move to ``critical contours,'' defined in section 7.  Then, Corollary \ref{cor:slant} allows us to interpret these critical contours as surface curves with important properties (1-cells of the slant MS complex), so that  we arrive at ``surface 1-cells.''  These surface 1-cells correspond to boundaries of qualitative parts of the surface (bumps, valleys, ridges, etc.) and function as a kind of scaffold on which the surface can be ``built.'' Various inpainting or diffusion algorithms could complete this scaffold of 1-cells back into a scalar field; for an example, see Figure \ref{fig:recon_1_cells}.  The completion is not unique, which relates directly to perception; the scaffold is the qualitative invariant on which different subjects build their quantitative percept. Thus we arrive at a surface exemplar in the upper right.

The qualitative nature of the solution we are proposing for shape perception bears some resemblance to the categorical parts and necks that can be inferred from the medial axis (or skeleton) \cite{kimia1995shapes, hung2012medial}, but our scheme focuses on interior lines and shading distributions rather than bounding contours. Importantly, just as the medial axis can be used to structure grasping of objects \cite{przybylski2011planning}, our critical contours may suffice for this as well. When reaching to grasp an object, we preform our hands to reflect the pose and extension of the handle \cite{jeannerod1995grasping, paulignan1997influence}.  For both people and robots, then, qualitative properties seem to suffice. In effect, we get global constraints by integrating local conditions on contours; evidence is beginning to accumulate that such segmentations induce psychophysical limits \cite{Koenderink15}.

\section{Image formation and assumptions}
\label{sec:image_formation}
We now describe the image formation process.  Consider an orthogonal projection model and define the 3D coordinate axis $(x, y, z)$ so that $(x, y)$ parametrizes the image plane and $z$ is the view direction.  Let $e_1, e_2, e_3$ represent the standard basis (unit vectors in these cardinal directions).  We think of the image $I(x, y)$ as being created from a ``cue'' or, more precisely, by applying a rendering function $F$ to the normal field $N(x, y) = (n_1, n_2, n_3)$ of a smooth surface $S(x, y)$.  That is,
\begin{align}
& F: \mathbb{S}^2 \rightarrow \mathbb{R}, \\
& I(x, y) = F(N(x, y)). \label{eqn:rendering_func}
\end{align}
Many familiar cues have this structure.  For example, Lambertian shading is equivalent to $F_L (N(x, y)) = \sum_i L_i \cdot N(x, y)$ for diffuse light sources $L_i$.  The spatial frequency cue of isotropic texture (once ideally blurred) is monotonically related to $F_T (N(x, y)) = e_3 \cdot N(x, y)$.  Specular shading is equivalent to $F_S (N(x, y)) = \sum_i \left[ (L_i - 2 (N \cdot L_i) N) \cdot e_3 \right]^p$ for specular light sources $L_i$ and some constant $p$.  We seek image contours that are always present independent of the choice of $F$.

For theoretical analysis, we consider $C^2$ rendering functions $F$ so that gradient fields are Lipschitz continuous.  The differential of  (\ref{eqn:rendering_func}) yields
\begin{align}
\label{eqn:differential}
DI(x, y) & = DF_{N(x, y)} \circ \bm{DN}(x, y) \\
& = DF^T \bm{DN}(x, y) \label{eqn:differential2} \\
& = \begin{bmatrix} F_x & F_y & F_z \end{bmatrix}  \begin{bmatrix}  n_{1, x}  & n_{1, y} \\  n_{2, x} &  n_{2, y} \\ n_{3, x}  & n_{3, y} \end{bmatrix}.
\end{align}

\noindent
To go from \eqref{eqn:differential} to \eqref{eqn:differential2}, note that function composition here is matrix multiplication.  To simplify notation, we also drop the point of application of $DF$.   (A comment on notation: we will use bold lettering to denote tensors and matrices, while vectors will be in standard lettering.  For an introduction to this tensor notation, see the appendix in \cite{tensors-arxiv-1}.)

Here, $DI$ is the 1-form corresponding to dot product with the gradient $\nabla I$. $DF$ is a $1 \times 3$ vector  and $\bm{DN}$ is a $3 \times 2$ matrix.   The image gradient orientation, that is, the angle of the vector $\nabla I$, is generally dependent on both the surface through the operator $\bm{DN}$ and the material/cue through the operator $DF$.   In the most general scenario, where $DF$ has no constraint, we cannot constrain $\bm{DN}$ from the data $DI$.  Thus, we will now put some weak limits on $DF$, $\bm{DN}$, and $\bm{D^2 N}$.

\subsection{Rendering function assumptions}
\begin{definition}
\label{rendering_funct_assum}
The admissible cue class is the set of differentiable rendering functions $F$ satisfying the following two criteria:
\begin{enumerate}\leftskip38pt
\item[{\rm (1)}] Bounded variation: there exists a $C_1$ such that for all rendering functions $F$ and $N_0 \in S^2,  || \nabla F || < C_1$.
\item[{\rm (2)}] $F$ is concave and $\bm{D^2 F}$ is bounded: There exists a constant $C_2$ s.t. $C_2 < \bm{D^2 F}_{N_0} (u, u)\break \leq 0$ for all $N_0 \in S^2, u \in T_{N_0} S^2$.
\end{enumerate}
\end{definition}

We elaborate on the conditions.  The bounded variation condition ensures that arbitrarily large changes in the image cannot be due to the rendering function alone but must also require some change in the normal field.  Without a constraint on the rendering function such as this one, we could not decipher between gradients due to material changes (such as a painting) and gradients due to natural shading changes.  There is perceptual evidence supporting this condition.  If an image feature is to be seen as ``shading,'' it must have generally low contrast.  Very high contrast features are often seen as material changes \cite{DanVSS13}.  The concave condition ensures that
if the unit sphere were imaged with a rendering function, we would see only one highlight (point of maximum brightness).
Note that this condition also relates to ``cloudy day'' \cite{langer1994shape, Langer99perceptionof} illumination, where the aperture function plays the role of, or is aided by, the surface normal; this also eliminates  \cite{horn1993impossible}.
Although our rendering function is quite general, it is designed to facilitate our analysis; some light field effects may not be included \cite{langer1997light, koenderink2003visual}.

Thus, a given rendering function $F$ creates an image $I(x, y)$ of a given imaged surface $S$.   Suppose we now choose a new rendering function $\tilde{F}$ (e.g., by changing the light source) to get a new image $\tilde{I}(x, y)$ of the same surface.  Our main theorem describes an important commonality between these images.  A registration correspondence exists between these two images, $I(x, y)$ and $\tilde{I}(\tilde{x}, \tilde{y})$, but we will not focus on describing or calculating it here.  Instead, to simplify analysis, we will regard the second image as a new scalar field on the same coordinate system: we will prove things regarding $\tilde{I}(x, y)$.

We now restrict to \emph{generic} interactions between the rendering function and imaged surface.

\subsection{Generic surface assumptions}
An image of a surface rarely completely describes the surface.  Since shape reconstruction is generally ill-posed, surfaces can collude with rendering functions to create images that hide surface features.  We seek to remove these rare cases via assumptions here.
        For example, in Lambertian shading, the image is a projection of the surface normal field onto an unknown direction (given by the light source), so variation in the normal field in directions perpendicular to the light source will have no effect on a local image patch: e.g., a Lambertian right cylinder cone with a light source directly above would create a constant intensity image.  In this case, $| D N|$ can be arbitrarily large while $| D I |$ is 0.  We wish to avoid cases like these.

        A slight change of light source or viewpoint would make the curvature of the cone visible and would therefore lead to large changes in the image.  Thus, we use the term ``generic'' to represent assumptions that remove rare or unstable configurations.  These unstable configurations often vanish with a small perturbation of scene parameters, i.e., light source or viewpoint changes.  There are two forms of generic that we will assume for our setup:
 \begin{enumerate}
 \item Given a curve $\alpha(t)$ on a surface $S$, we assume that the three column vectors of the unfolded tensor $D^2_{\alpha(t)} N$,
 \[
 \{D^2N (u, u),  D^2 N (u, w), D^2 N (w, w)\} \in \mathbb{R}^3,
 \]
 contain at least two that are linearly independent.
 \item Let $\theta(v_1, v_2)$ represent the acute angle between two vectors in $\mathbb{R}^3$.  Given a curve $\alpha(t)$ on a surface $S$ and rendering function $F$, we assume that there exists an $\epsilon$ such that for all $t$,
 \begin{align}
 \theta(DF, DN_{\alpha(t)} (\cdot) )& > \epsilon, \\
  \theta(DF, D^2 N_{\alpha(t)} (u, u)) & > \epsilon, \\
  \theta(DF, D^2 N_{\alpha(t)} (w, w)) & > \epsilon, \label{gen_assumptions} \\
 || D F || & > \epsilon. \
 \end{align}
 That is, the rendering function's differential does not happen to align along certain differential properties of the surface normal field.  Many of these properties are arbitrarily small measure in the space of continuous configurations.  (This removes the Lambertian cone example.)  Experimentally, violations of these conditions are rare.

 \end{enumerate}

Of course, there are other obstacles to 3D shape reconstruction that we are ignoring, such as multiple scattering, partial occlusions, or textured objects.  We are instead focused on understanding the stability and geometric meaning of image contours derived from shading.  We now define this relationship.

\section{The contour interpreted as a shading limit}
 \label{sec:contour_is_shading}
\begin{figure}[!b]
\begin{center}
a) \includegraphics[trim= 0cm 0cm 0cm 0cm, clip=true, width = 0.33 \linewidth]{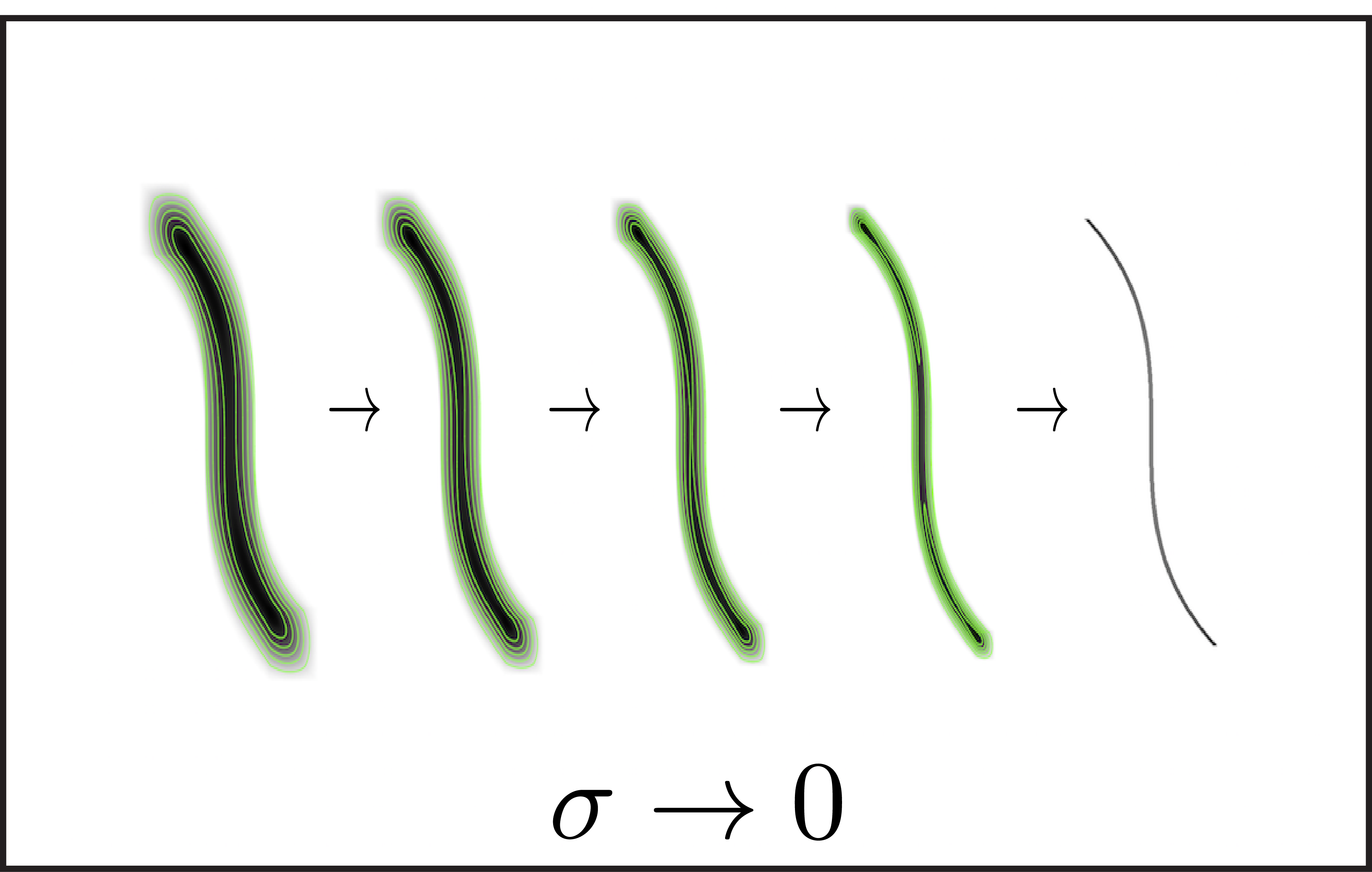} 
b) \includegraphics[trim= 2cm 1cm 2cm 1cm, clip=true, width = 0.18 \linewidth]{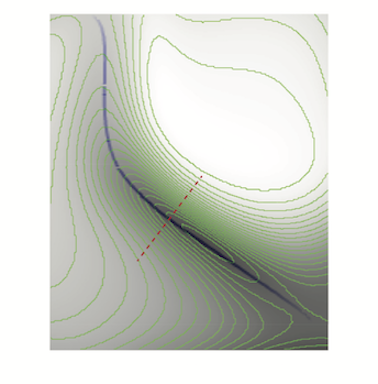}
c) \includegraphics[trim= 1cm 1cm 1cm 1cm, clip=true, width = 0.23 \linewidth]{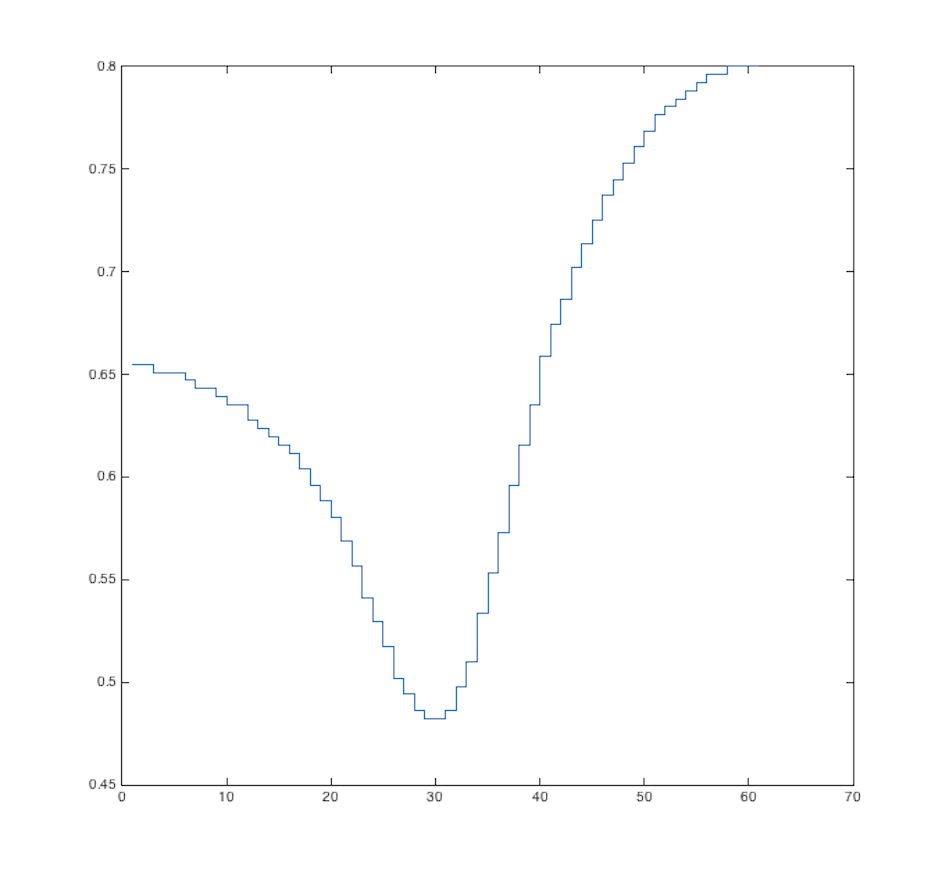}
\caption{We inherit the meaning of a contour by considering it as a limit of shaded images. {\rm (a)} From left to right, we start with a blurred version of the lower left contour from Figure {\rm \ref{fig:blob}}, drawn in blue.  Isophotes are in green.  As we remove the blurring, we represent an element of the shading sequence $I^\sigma$ of the contour with successively smaller $\sigma$. {\rm (b)} We show isophotes from a Lambertian shaded image of the surface.  Note the similarity of the two isophote patterns near the contour: on either side of the contour the direction of the level curves rotates to be nearly tangential to the contour.  Also, note how the contour (blue) is nearly a gradient flow of the shaded image.  The dotted red line represents a transversal direction, with plotted pixel values in {\rm (c)}.  Note the steep local shading minimum across the contour; this relates to the definition of height ridge found in~{\rm \cite{lindeberg1998edge}}.}
\label{fig:transversal_example}
\end{center}
\end{figure}

We seek a visual pattern that is present across many views and many renderings.   We are inspired by artist sketches, where a collection of thin strokes on paper inspires a surface perception.  Eventually we will think of these strokes as a robust skeleton for describing part boundaries implied by a shaded image.  To understand what is the physical meaning (constraint on the viewed surface) inherent in each stroke, we start by investigating their differential properties.

Consider a line drawing image as a collection of 1D contours.  Focusing on one contour, $\alpha(t)$, assume it has bounded image (planar) curvature.  For each point $\alpha (t)$, we have a scalar intensity value $I(\alpha (t)) > 0$; we require this value to be 0 at the endpoints.  Without loss of generality, let $\alpha(t)$ be arc-length parametrized.

\begin{definition}
  An ideal {\rm 1}D contour $\alpha(t)$, $0 \leq t \leq 1$, can be expressed as a scalar field $I_{\alpha} \subset \mathbb{R}^2$ in the following way:
\begin{align}
\label{ideal_contour}
I_{\alpha} (x,y)  & = I(\alpha(t)) \quad \forall \, t \in (0, 1) \, s.t. \, \alpha(t) =  (x, y),\\
I_{\alpha} (x,y) & =        0 \quad \quad \quad \text{    otherwise}. \
\end{align}
\end{definition}
%

We wish to understand the behavior of the image derivatives of $I_\alpha$ but, since $I_\alpha$ is discontinuous, the derivatives do not exist.  However, we can approximate these derivatives by considering $I_\alpha$ as the limit of a sequence of shaded images $\{ I^\sigma \}$ as $\sigma \rightarrow 0$ on a tubular neighborhood $\Omega_\alpha$.  We define each shaded image $ I^\sigma$ as a convolution with Gaussian functions $G(\sigma)$ of $I_\alpha$ with successively smaller standard deviation $\sigma$ in the following manner.

For every point $p$ on $\alpha(t)$, we parametrize the local neighborhood $U_\alpha(t)$ with two directions.  For convenience, write $\bm{u}(t) = \alpha'(t)$.  Define $\bm{w}(t)$ to be the transversal direction at the point $p$, so $\bm{w}(t) \cdot \bm{u}(t)= 0$. $\{ \bm{u}, \bm{w} \}$ is an orthonormal basis. Let $\xi, \eta$ be the corresponding coordinate functions.  As we are only interested in the limiting behavior and as $\alpha(t)$ has bounded curvature, we realign the frame so that $p$ is at the origin and define
\begin{align}
\label{eqn:shading_limit}
I^\sigma (p + \xi \bm{u} + \eta \bm{w}) & = \frac{1}{\sqrt{2 \pi} \sigma} \int I_\alpha (p + x \bm{u}) e^{-\frac{\eta^2}{2 \sigma^2}}    e^{-\frac{(x - \xi)^2}{2 \sigma^2}} \, dx \
\end{align}

We can now calculate image derivatives of $I^{\sigma}$ in the $\bm{u}, \bm{w}$ directions (see the appendix) with the results summarized below and illustrated in Figure~\ref{fig:transversal_example}.

\begin{lemma}
\label{lemma:shading_derivs}
Let $\alpha(t), I_\alpha, \bm{u}, \bm{w}$ be defined as above.  The sequence of shaded images $\{ I^\sigma \}$ converge pointwise to the original line drawing $I_\alpha$ and have the following properties on the derivatives  as $\sigma \rightarrow 0$:
\label{conditions}
\begin{enumerate}\leftskip38pt
\item[{\rm (1)}] $I^\sigma_{ww} (\alpha(t)) \rightarrow -\infty$ for every $t \in [0, 1]$.
\item[{\rm (2)}] $I^\sigma_{uu} (\alpha(t)) \rightarrow \infty$ for $t = \{0, 1\}$.
\item[{\rm (3)}] There exists a constant $M$ such that  $| I^\sigma_{w} |, |I^\sigma_{u}|, |I^\sigma_{uw}| < M$ for every $t \in [0, 1]$.
\end{enumerate}
\end{lemma}

The first and most important condition implies that the contour sides ``pinch in'' as $\sigma \rightarrow 0$.  Thus an ``ideal contour'' can be seen as pointwise close to a shading pattern with the above derivatives.  This leads us to define \emph{critical contours} in the next section, which are nearly invariant shading patterns that mimic these artist's strokes; afterward we will connect  such critical contours to MS complexes.







\section{Critical contours}
We now define a \emph{critical contour}, the visual pattern that is (nearly) invariant across the admissible rendering class defined above.  This critical contour will have image derivatives similar to those calculated in Lemma \ref{lemma:shading_derivs} for the ideal contour.

\begin{definition}
\label{defn_cc}
A $K$-critical contour $(\alpha(t), I(x, y), M, K)$ is a curve $\alpha(t)$ on an image $I(x, y)$ such that the following conditions hold for all $t$:
\begin{enumerate}\leftskip38pt
\item[{\rm (1)}] $| I_{ww} (\alpha(t))| > K$ for every $t \in [0, 1]$,
\item[{\rm (2)}] $| I_{uu} (\alpha(t))|  > K $ for $t = \{0, 1\}$,
\item[{\rm (3)}] $| DI |, |I_{uw}| < M$ for every $t \in [0, 1]$,
\end{enumerate}
\noindent for positive $M, K \in \mathbb{R}$.
\end{definition}

For the remainder of the paper, let $\alpha(t)$ denote a $K$-critical contour with the conditions from Definition \ref{defn_cc}.
As $K \rightarrow \infty$, \textit{K}-critical contours converge pointwise to the ideal contour defined in (\ref{ideal_contour}).
In Theorem \ref{main_theorem}, we show these \textit{K}-critical contours are also 1-cells of the MS complex in any image obtained from a rendering function in our admissible class. In general, \textit{K}-critical contours are ``stronger'' 1-cells that persist if the rendering function is changed.   Note the $|I_{ww}| > K$ condition above is stronger than the usual condition for intensity to be at a transversal maximum for differential geometric ridges \cite{haralick1983ridges}.

\begin{theorem}$1.$
\label{main_theorem}
Let $F, \tilde{F}$ be any two rendering functions in the admissible cue class.  Applying these rendering functions to a generic surface $S$, we obtain two corresponding images $I(x, y), \tilde{I}(x, y)$.  For any $\epsilon > 0$, there exists a $K \in \mathbb{R}$ such that the surface region corresponding to an
$\epsilon$-neighborhood of  a $K$-critical contour in $I$ contains an MS $1$-cell for image $\tilde{I}$.
\end{theorem}

To gain intuition for Theorem \ref{main_theorem}, consider the surface $f(x, y) = a x^2 + b y^2 + c$. Note that $(0, 0)$ is a critical point of $f$, and think of $f$ as a height function above a plane with normal vector $(0, 0, 1)$.  Now, define $g$ to be another height function from the surface defined by $f$ to a different plane with normal vector $(g_1, g_2, g_3)$.   In general, $(0, 0)$ is not a critical point for $g$.   However, if $|a|$ and $|b|$ is large enough, then $g$ will have a critical point arbitrarily close to $(0, 0)$.  Thus, $f$ and $g$ almost ``share'' critical points.
\begin{figure}[t]
\begin{center}
\includegraphics[trim= 5cm 0cm 5cm 0cm, clip=true, width = 0.4 \linewidth]{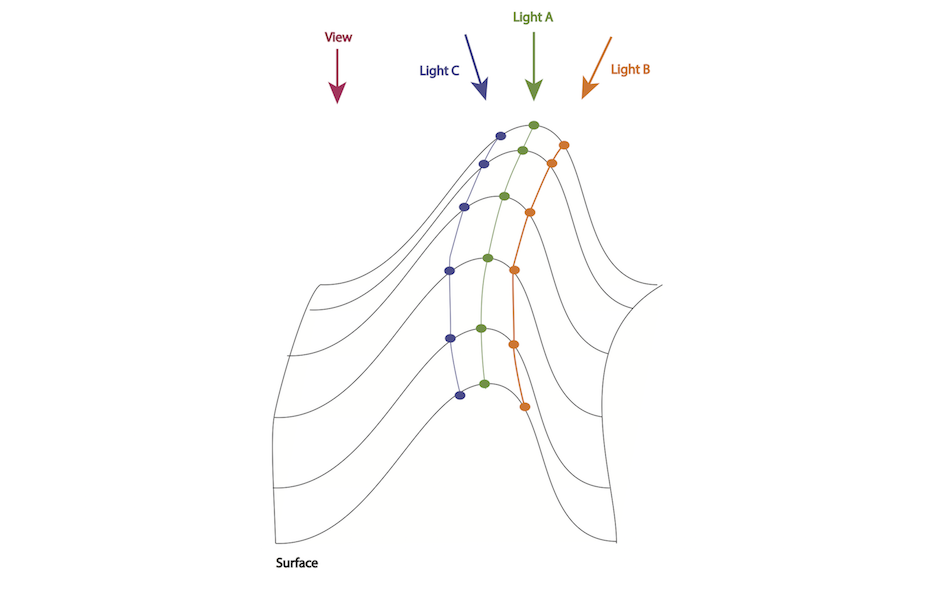}
\caption{In a surface region with anisotropic curvature (the two principal curvatures differ vastly), the image gradient flows are robust as we change the light source.  Each curve represents an MS $1$-cell (critical contour) on the image corresponding to the light source with the same color.  As we move the light source from A to B, the integral path shifts a small amount.}
\label{fig:shifting_ls}
\end{center}
\end{figure}


\begin{figure}[t]\vspace*{-12pt}
\begin{center}
a) \includegraphics[trim= 2cm 1cm 2cm 0cm, clip=true, width = 0.3 \linewidth]{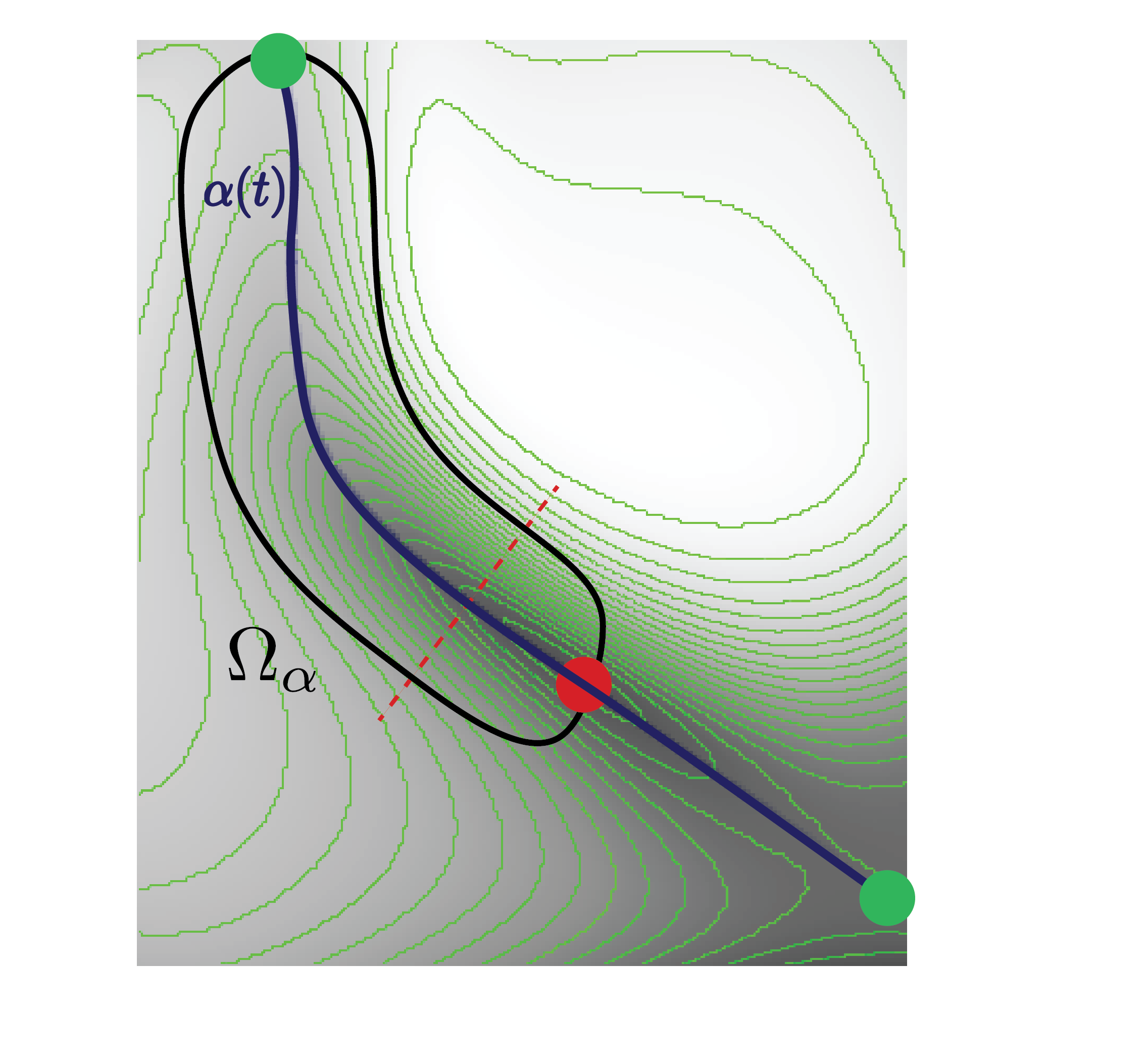}
b) \includegraphics[trim= 0cm 0cm 0cm 2cm, clip=true, width = 0.5 \linewidth]{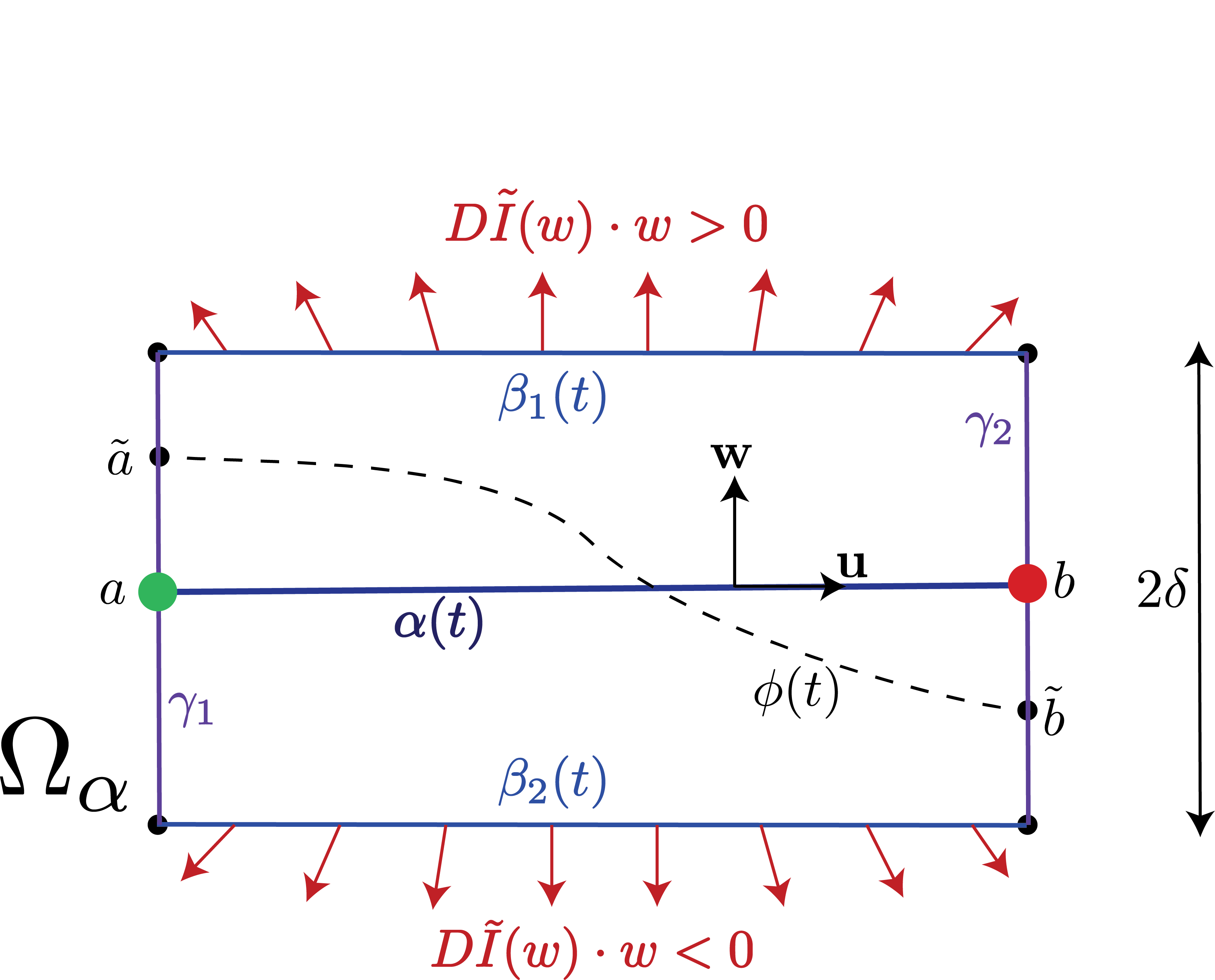}
\caption{We have drawn the tubular neighborhood $\Omega$ with critical contour $\alpha(t)$ as a straight line for simplicity.  We show there exists an MS $1$-cell $\phi(t)$ of the new image $\tilde{I}$ given the conditions in Definition {\rm \ref{defn_cc}} on $\alpha(t)$.  Lemmas {\rm \ref{lemma:crit_points} and \ref{lemma:beta_curves}} show the existence of $\tilde{a}, \tilde{b}$ and $\beta_1(t), \beta_2(t)$ with the property that $D \tilde{I}$ is outward facing, shown as the red vectors pointing outside $\Omega$.  Theorem {\rm \ref{main_theorem}} concludes that there must be an integral path of $\tilde{I}$, shown as $\phi(t)$, inside $\Omega$ and this must be a $1$-cell of $\tilde{I}$. \label{fig:box_diagram}}
\end{center}
\end{figure}

From the above example, it is plausible to believe that if certain curvatures of a surface $S$ are large enough, scalar fields on $S$ resulting from different projections of the normal field may share critical points.  We generalize this to find when they also share MS 1-cells, a 1D analogue to critical points.  See Figure \ref{fig:shifting_ls}.  We now need to show that the presence of a critical contour implies sufficient curvature across the contour to support the above intuition.

The proof will follow in three steps (Figure~\ref{fig:box_diagram}).   We will show that the presence of a $K$-critical contour $\alpha(t) \subset I(x, y)$ implies the following ``$\epsilon$-box structure'' on $\Omega_\alpha$ in the unknown image function $\tilde{I}(x, y)$:

\begin{enumerate}
\item  In Lemma \ref{lemma:crit_points}, we show there exist critical points $\tilde{a}, \tilde{b} \in \tilde{I}(x, y)$ that are $\epsilon$-close to the endpoints $a, b$ of $\alpha(t)$.

\item In Lemma \ref{lemma:beta_curves}, we show there exists two curves $\beta_1(t), \beta_2(t)$ in the $\delta$-tubular neighborhood where the gradient $\nabla \tilde{I}$ points away from $\alpha(t)$.  These two lemmas give the ``$\epsilon$-box structure'' on $\Omega_\alpha$.

\item In Lemma \ref{lemma:flow_dir}, we show $\Omega_\alpha$ contains a $\tilde{I}$ 1-cell, shown as $\phi(t)$.  This is proven by first showing, without loss of generality, that all integral paths flow from left to right.  Then, either $\tilde{a}$ or $\tilde{b}$ must be a saddle and there is an integral path that traverses $\Omega$ connecting to at least one of them, proving Theorem \ref{main_theorem}. 

\end{enumerate}



We now state Lemmas \ref{lemma:crit_points} and  \ref{lemma:beta_curves}.  Their proofs involve technical calculations via Taylor approximations, so we leave them to the appendix.

\begin{lemma}
\label{lemma:crit_points}
Let $a$ be an endpoint of $\alpha(t)$. Given a new rendering function $\tilde{F}$, resulting image $\tilde{I}$, and any $\delta > 0$, there exists a $K_0$ such that if $K > K_0$, the following holds:  $ \exists \, \,\tilde{a} \,\in \Omega_\alpha,$ such that $|| a - \tilde{a} || < \delta$ and $\tilde{a}$ is a critical point of $\tilde{I}$.
\end{lemma}

\begin{lemma}
\label{lemma:beta_curves}
Recall that $w(t)$ is the transversal direction in $\mathbb{R}^2$ to $\alpha(t)$.  Given a new rendering function $\tilde{F}$, resulting image $\tilde{I}$, and $\delta > 0$, there exists a $K_0$ such that if $K > K_0$, the following holds:
Define two curves $\beta_1(t) =  \alpha(t) + \delta w$ and  $\beta_2(t) = \alpha(t) -\delta w $.   On $\beta_1(t)$, $D \tilde{I}_{\beta_1(t)} (w(t)) > 0$ and on $\beta_2(t)$, $D \tilde{I}_{\beta_1(t)} (w(t)) < 0$.

\end{lemma}

These two lemmas prove that the vector field $D \tilde{I} (w)$ behaves as shown in red in Figure \ref{fig:box_diagram}(b) and that stationary points of $D \tilde{I}$ are at $\tilde{a}$ and $\tilde{b}$.  It remains to show that this vector field constraint implies the integral line $\phi(t)$ for $\tilde{I}$ inside $\Omega_\alpha$.

\begin{lemma}
\label{lemma:flow_dir}
Let $\alpha(t)$ be a $K$-critical contour with the conditions from Definition {\rm \ref{defn_cc}}.   Given a new rendering function $\tilde{F}$, resulting image $\tilde{I}$, and $\delta > 0$, apply the previous two lemmas.  We can find critical points $\tilde{a}, \tilde{b}$ of $\tilde{I}$ and two curves $\beta_1(t), \beta_2(t)$ arbitrarily close to $\alpha(t)$, as illustrated in Figure {\rm \ref{fig:box_diagram}}.  Parametrize two line segments $\gamma_1(s), \gamma_2(s)$ with the following properties:
\begin{align*}\gamma_1(0) &= \beta_1(0), \gamma_1(1) = \beta_2(0), \gamma_1(0.5) = \tilde{a},
\\
\gamma_2(0) &= \beta_1(1), \gamma_2(1) = \beta_2(1), \gamma_2(0.5) = \tilde{b}.
\end{align*}

Define the region $\Omega_\alpha$ as that bounded by the curves $\{\gamma_1, \beta_1, \gamma_2, \beta_2\}$.  Without loss of generality, every integral path of $\tilde{I}$ that intersects $\Omega_\alpha$ enters from a point on $\gamma_1$ and leaves on a point on $\gamma_2$.

\end{lemma}
\begin{proof}
First, we assume that there are no critical points of $\tilde{I}$ in the interior of $\Omega_\alpha$.  If there are, bisect $\Omega_\alpha$ into $\Omega_1$ and $\Omega_2$ and repeat the following argument.

Let $S_1$ be the set of all points in $\Omega_\alpha$ on integral curves entering from points on $\gamma_1$.  We say that an integral path $P$ enters from $\gamma_1$ when there exists an ${r, s}$ such that $\gamma_1(r) = P(s)$ and $P'(s) \cdot \alpha'(0) > 0$.  Let $S_2$ be the set of all points in $\Omega$ on integral curves entering from points on $\gamma_2$.

Clearly, $\Omega_\alpha = \bar{S_1} \cup \bar{S_2}$.  It suffices to show that one of the $S_i$ is empty.  Suppose not; suppose $S_1 \neq \varnothing, S_2 \neq \varnothing$.  Being a tubular neighborhood of a curve $\alpha(t)$, $\Omega_\alpha$ is a topologically connected space in $I$.  Thus, $\bar{S_1}$ and $\bar{S_2}$ must not be disjoint.  There exists a point $p \in \bar{S_1} \cap \bar{S_2}$.  As there are no critical points in $\Omega_\alpha$, $\nabla \tilde{I} (p) \neq 0$.  For any $\epsilon$, there exists an $\epsilon$ neighborhood of $p$ containing both an integral path $\psi_1 \subset S_1$ and an integral path $\psi_2 \subset S_2$.  However, an integral path is the solution to a differential equation $\psi'(t) = \nabla I_2 (\psi(t))$ with initial condition $\psi(0) = q$.  For a Lipschitz continuous gradient field, there is continuous dependence of solutions on the initial conditions $\psi(0)$ \cite{Sastry99}.   Thus, $\psi_1$ and $\psi_2$ must be arbitrarily close together, which yields a contradiction, as they go through points on opposite sides of $\Omega_\alpha$.
\end{proof}

We now have all the pieces to prove the main theorem. It remains to show that, given the conditions in the above lemma, there is an MS 1-cell of $\tilde{I}$ contained in $\Omega_\alpha$.

\begin{proof}[Proof of Theorem {\rm \ref{main_theorem}}]
From Lemma \ref{lemma:flow_dir}, we see that all integral lines flow from a point on $\gamma_1$ to a point on $\gamma_2$ or vice versa.  $\tilde{a}$ is a critical point on $\gamma_1$ and $\tilde{b}$ is a critical point on $\gamma_2$.  As the flow direction on $\beta_1, \beta_2$ points outward for all $t$, the critical index of $\tilde{a}$ and $\tilde{b}$ can only differ by at most 1.  Without loss of generality, $\tilde{b}$ has an incoming integral path $\phi$ starting from the other side of $\Omega_\alpha$.  Thus $\tilde{b}$ must be a saddle point and $\phi$ must be an MS 1-cell traversing $\Omega_\alpha$.~~~
\end{proof}

\begin{corollary}
As $K \rightarrow \infty$, as in the case of our ideal contour in Lemma {\rm \ref{lemma:shading_derivs}}, a $K$-critical contour  $\alpha(t)$  in any admissible image represents an MS $1$-cell in any other admissible image.
\end{corollary}

\begin{proof}
As $K \rightarrow \infty$, the tubular neighborhood $\Omega_\alpha$ of $\alpha(t)$ shrinks to zero width.  The integral path $\phi(t)$ must traverse $\Omega_\alpha$ and thus must eventually lie on $\alpha(t)$.
\end{proof}

This means that an ideal contour represents a visual commonality among all images of the surface $S$.  As the normal slant function is a member of our admissible rendering functions, an ideal contour also lies infinitely close to a surface property: an MS 1-cell of the slant function.  A decomposition of the slant function into stable and unstable manifolds via its MS 1-cells is a representation of the surface (very similar to a concave/convex representation) that we are investigating further.  Thus, we can now interpret an ideal image contour as a surface property that ``shines through'' in every image created by any of the rendering functions.

\begin{corollary}
For $K$ sufficiently large, a $K$-critical contour $\alpha(t)$ in image $I$ of surface $S$ aligns with an MS $1$-cell of the slant of the surface normal field of $S$.
\label{cor:slant}
\end{corollary}

\begin{proof}
Define $\tilde{F}(N)$ = $\langle e_3, N \rangle$ as the rendering function corresponding to a Lambertian surface with light source in the view direction $e_3$.  Define $\tilde{I}$ as the image associated with the surface $S$ using this rendering function.  As the slant of the normal field is a monotonic function of the image $\tilde{I}$, it shares the same MS complex as $\tilde{I}$.  Apply the theorem to show that $\alpha(t)$ aligns with an MS 1-cell of $\tilde{I}$.
\end{proof}

\enlargethispage{-3pt}

\section{The Morse--Smale complex on shading and slant}

We now apply the above theory to a number of different shapes to illustrate how critical contours computed from a shaded image
relate to the MS complex on the slant function of the surface. The results are ordered in complexity and correspond to
Figures \ref{fig:rot_sigmoid_MS}, \ref{fig:furrow_MS}, and \ref{fig:blob_MS}.

\textit{A note on methodology}.
A 3D mesh was generated for each figure, which was then rendered under different conditions to produce each image.  We use \cite{Reininghaus11} to calculate the MS complex and consider persistence simplifications with few critical points from these images.  Alternate ways to simplify the MS complex in a more salient manner are \cite{Sahner08, Weinkauf09}.  We experimentally verified that MS 1-cells with large $I_{ww}$ remain positionally stable across these images, as predicted by our theorem.  We observe that, because the computations are run directly on quantized pixel values, there are certain numerical issues. We do believe that results could be further improved, while also generalizing to nonsmooth images, by computing on oriented filter responses instead.

The first example (Figure~\ref{fig:rot_sigmoid_MS})  consists of a large bump which, as the light source moves, illustrates the common critiques of flow-based approaches:
large movement in the isophotes (and in the location of the maximum in intensity). Notice, however, that two of the MS 1-cells (blue curves) form a circle and remain fixed surrounding the bump: these have large $K$; i.e., these lie along the large bright-dark-bright transitions and are the critical contours. Note that there are other 1-cells and critical points; these do not satisfy the $I_{ww}$ condition and so are irrelevant to the shape representation.

The second example is the furrow shape (Figure~\ref{fig:furrow_MS}) shown from two views and with drastically different lightings. Notice how the isophotes move, how the
maxima move, but how the critical contours remain stable.


The next example consists of images of a ``blob'' shape (Figure~\ref{fig:blob_MS}) constructed from random perturbations of a sphere. Notice again the stability of the critical contours, how these agree across lightings and for the slant function, and how these stable 1-cells correspond to
the suggestive contours that were computed from the true 3D shape.

\begin{figure}[t]
\begin{center}
\hspace{3cm} \includegraphics[trim= 0cm 9cm 0cm 9cm, clip=true, width = 0.3 \linewidth]{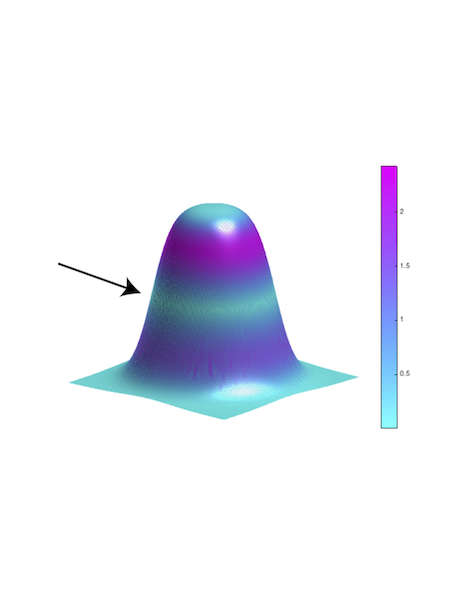} \newline
\includegraphics[width = 0.225 \linewidth]{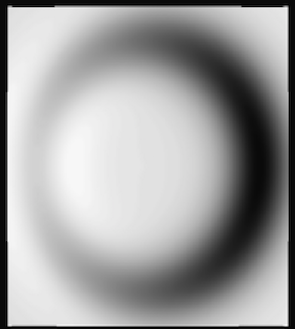}
\includegraphics[width = 0.225 \linewidth]{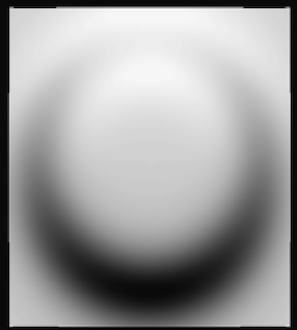}
\includegraphics[width = 0.235 \linewidth]{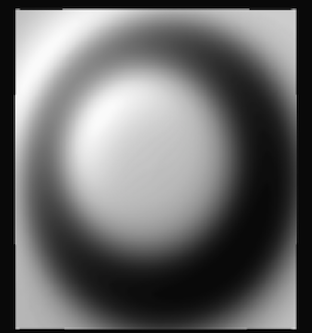}
\includegraphics[width = 0.228 \linewidth]{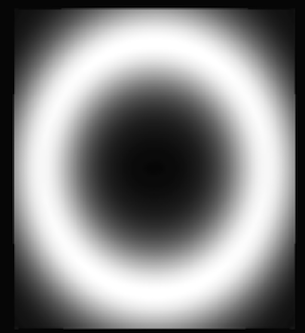} \newline
\includegraphics[trim = 17cm 3cm 15cm 3cm, clip=true,width = 0.23 \linewidth]{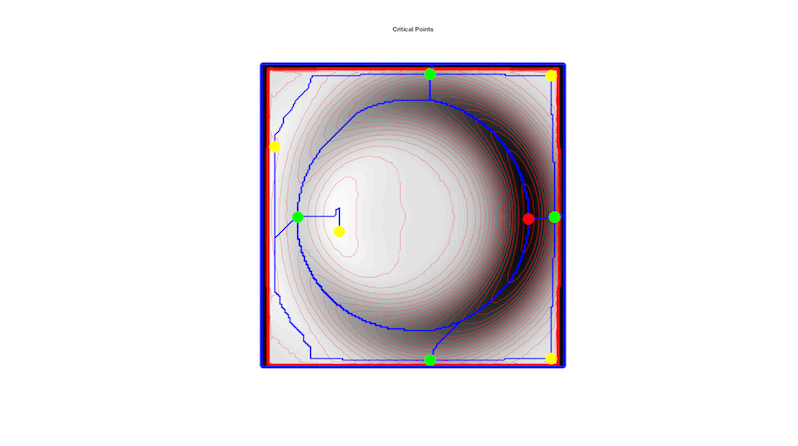}
\includegraphics[trim = 17cm 3cm 15cm 3cm, clip=true,width = 0.23 \linewidth]{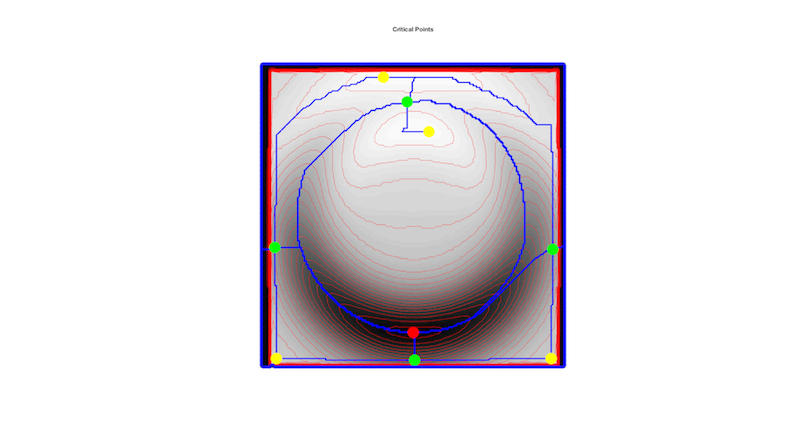}
\includegraphics[trim = 17cm 3cm 15cm 3cm, clip=true,width = 0.23 \linewidth]{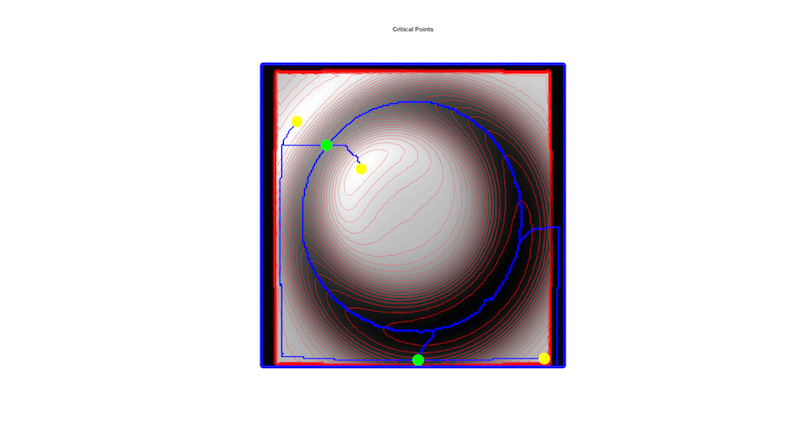}
\includegraphics[trim = 17cm 3cm 15cm 3cm, clip=true,width = 0.23 \linewidth]{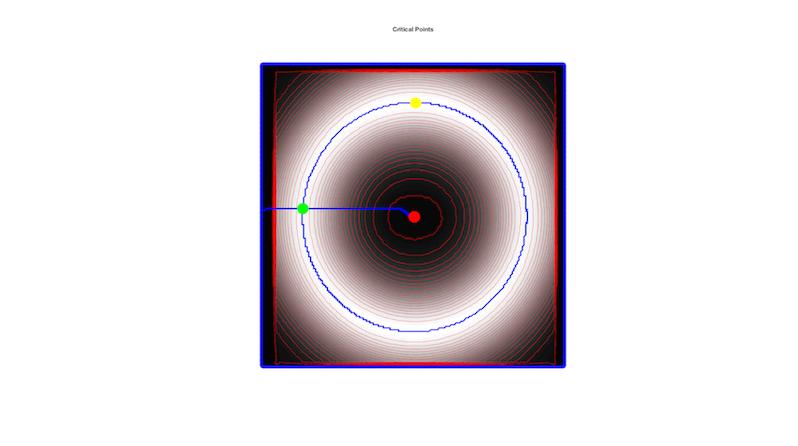}
\caption{Top: A slightly perturbed sigmoid rotated around the z-axis.  The color indicates the absolute value of the Gaussian curvature.  An arrow points to a turquoise band which is centered along a contour of near zero Gaussian curvature; this is a $1$-cell of the MS complex of the slant function.  Under a wide class of rendering functions (as described above), the resulting shaded image will contain a $1$-cell along this band.
First row: From left to right, the first two images are Lambertian shaded renderings of the above rotated sigmoid with different light sources.  The third image is a specular rendering.  The fourth image is the slant function. \label{fig:rot_sigmoid}
Second row: Corresponding MS complexes to the images above along with isophotes in red.  Blue arcs correspond to the $1$-cells. Yellow, green, and red points correspond to maximum, saddle, and minimum critical points.  Notice the blue common circular contour (which consist of unions of $1$-cells). \label{fig:rot_sigmoid_MS}}
\end{center}
\end{figure}

\begin{figure}[t]\vspace*{-2pt}
\begin{center}
\includegraphics[width = 1\linewidth]{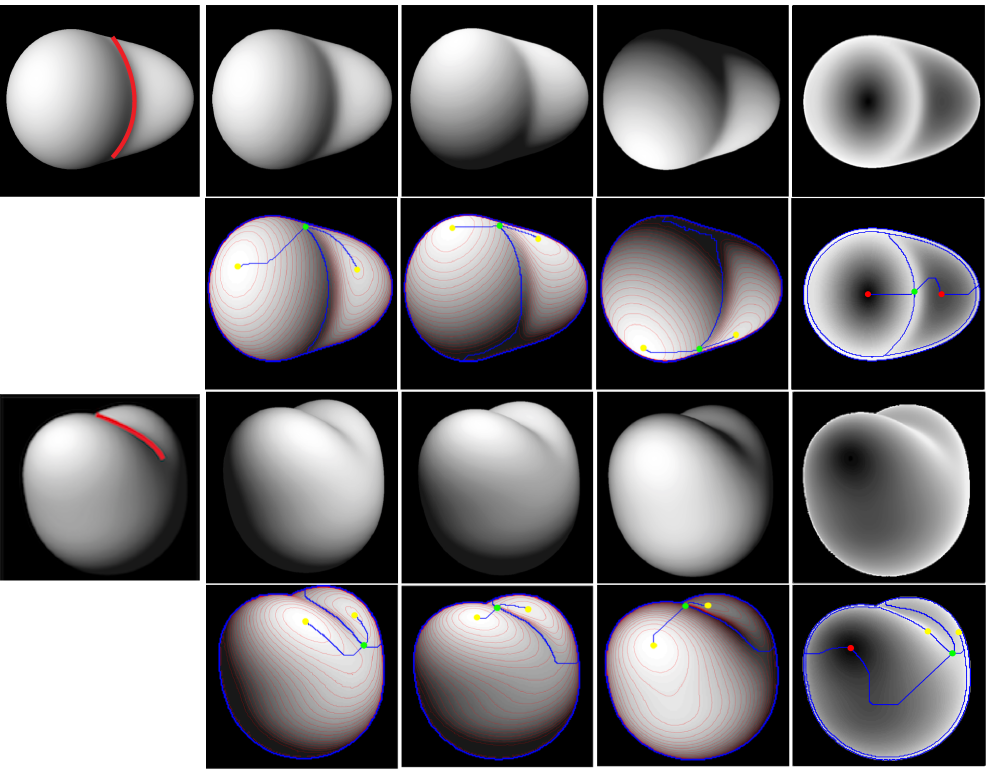}\vspace*{-6pt}
\caption{A second example that shows the commonality between some $1$-cells over both large light source changes and large view changes.  Row $1$, first column: A ``furrow shape'' image with a sketched red contour showing the critical contour.  Row $1$, columns {\rm 1--3:}  The furrow shape lit from three directions.  Row $1$, column {\rm 4:} True slant.  Row {\rm 2:} The MS $1$-cells with critical points (minima in red, saddles in green, maxima in yellow) corresponding to the images in the first row.  Third and fourth rows: Analogous to rows $1$ and $2$, with a different viewpoint.}\vspace*{-3pt}
\label{fig:furrow_MS}
\end{center}
\end{figure}

\begin{figure}[t]
\begin{center}
\includegraphics[width = 1 \linewidth]{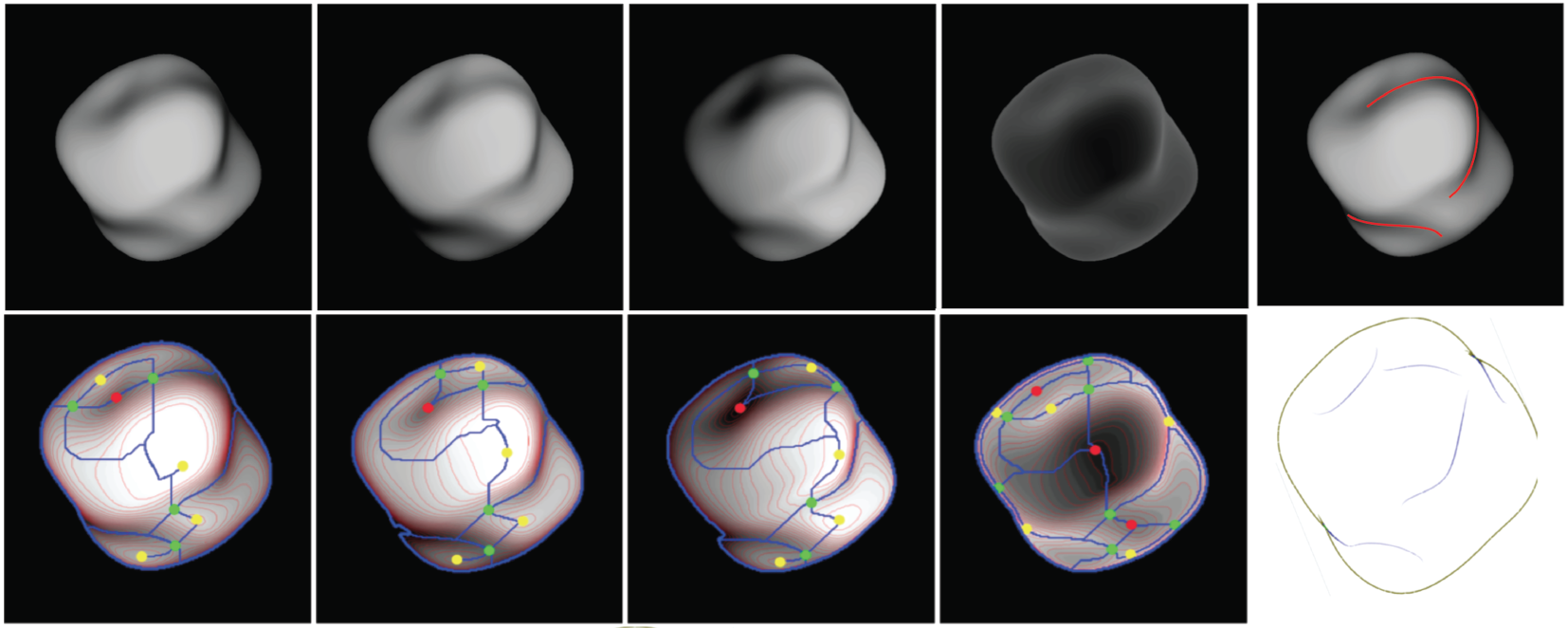}
\caption{First row, columns {\rm 1--3:} Lambertian images of a blob under different rendering functions.  Light source differences across images are at least $45$ degrees. First row, column {\rm 4:} The true slant function.  First row, column {\rm 5:} Sketched critical contours.
Second row, columns {\rm 1--4:} Corresponding MS complexes to above images.  Second row, column {\rm 5:} For contrast, the suggestive contours (in perspective projection) {\rm \cite{DeCarlo03}} for the same surface.  The extra suggestive contour in upper right is not seen in orthographic projection.
\label{fig:blob_MS}}
\end{center}
\end{figure}
In our next example, we experimentally verify Corollary \ref{cor:slant}.   In Figure \ref{fig:horse_compare_MS}, we overlay the MS complex for the horse image with the MS complex of the slant field.  Note the correspondence between the red segmentation, blue segmentation, and suggestive contours, as predicted by our theory.  On those curves where the two MS complexes are not in exact alignment, the value of $K$ is not sufficiently large.  This indicates where the qualitative structure of the slant (of the normal field) can be immediately and robustly inferred from a shaded image via the MS complex.

\begin{figure}[t]
\begin{center}
a) \includegraphics[trim = 5cm 2cm 5cm 2cm, clip=true, width = 0.3 \linewidth, height = 0.27\linewidth]{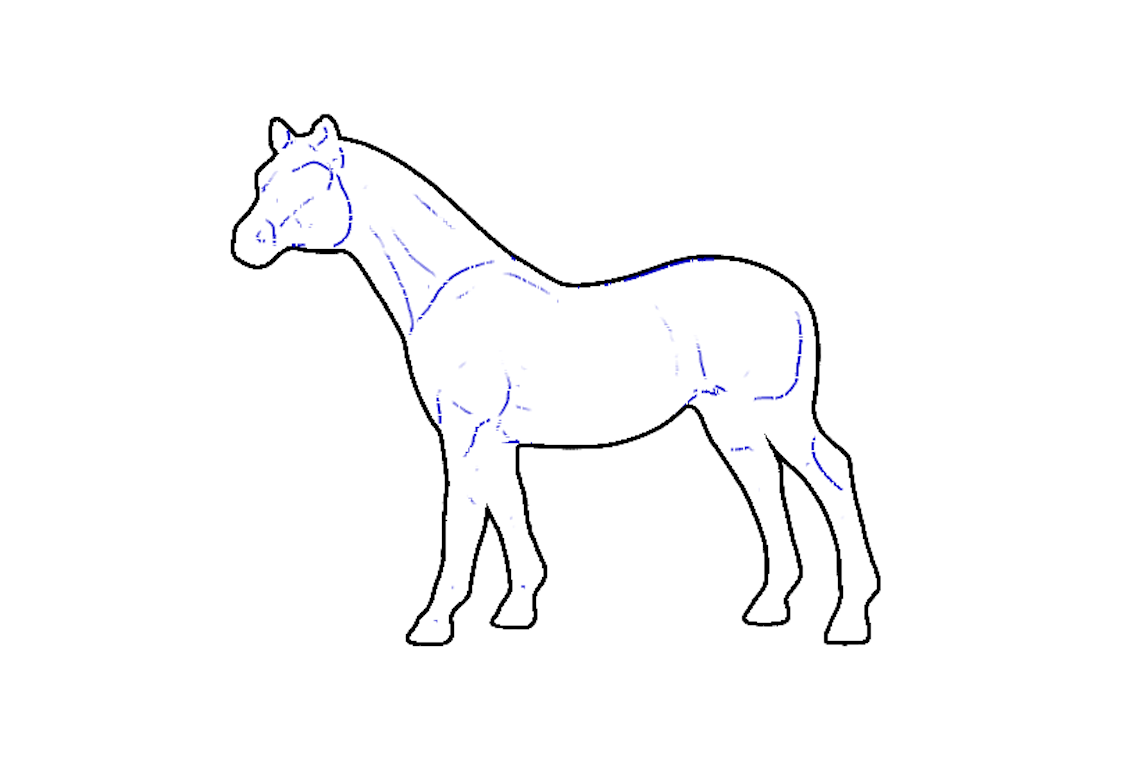}
b) \includegraphics[ width = 0.3 \linewidth, height = 0.3 \linewidth]{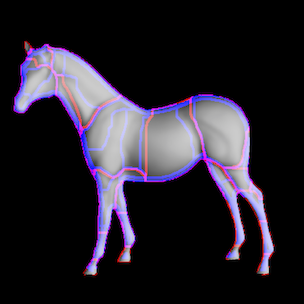}
\caption{Ideal line drawings, as modeled by suggestive contours {\rm \cite{DeCarlo03}}, relate to the MS complex of $1$-cells of both the image and slant scalar functions.  In particular, lines are often drawn at the 1D intersections of these two MS complexes.  {\rm (a)} The suggestive contours by {\rm \cite{DeCarlo03}}. {\rm (b)} The persistence simplified $1$-cells of the image function (red) and persistence simplified $1$-cells of the slant function (blue).  In cases where the two MS complexes don't exactly align, the value of $K$ is not large enough.}
\label{fig:horse_compare_MS}
\end{center}
\end{figure}

A consequence of the global nature of the MS complex is that it provides a qualitative solution
that segments the surface into salient parts as in Figure~\ref{fig:recon_1_cells}(b). There is a maximum on each of the four primary lobes, plus
several others. The part regions surrounding these are delimited by MS 2-cells, as are the interior (less reliable) 2-cells.  It is these interior 2-cells
that will shift with the light sources.


A remaining question is how to quantitatively reconstruct a scalar field from only knowledge of its critical contours.  This question, for the complete 1-cells, has been considered, for example, in \cite{Giorgis15, Weinkauf10}.  In Figure \ref{fig:recon_1_cells}(b), we show a simple example for the furrow object that the segmentation induced by 1-cells of the MS complex can be sufficient for a qualitative understanding of the slant.  We used the results from row 1 of Figure \protect \ref{fig:furrow_MS} and diffused the slant value from the occluding contour (where the slant is  $\pi/2$) onto the critical contour.  Then, we applied an inpainting algorithm \cite{DErrico} to ``reconstruct'' the scalar slant field.  We admit that this is a simple example, but in \cite{Gerber10}, one can see more complex examples of how the graph structure of the MS complex can capture the essential phenomena of
real-world data.



\begin{figure}[t]
\begin{center}
a) \includegraphics[width = 0.25 \linewidth]{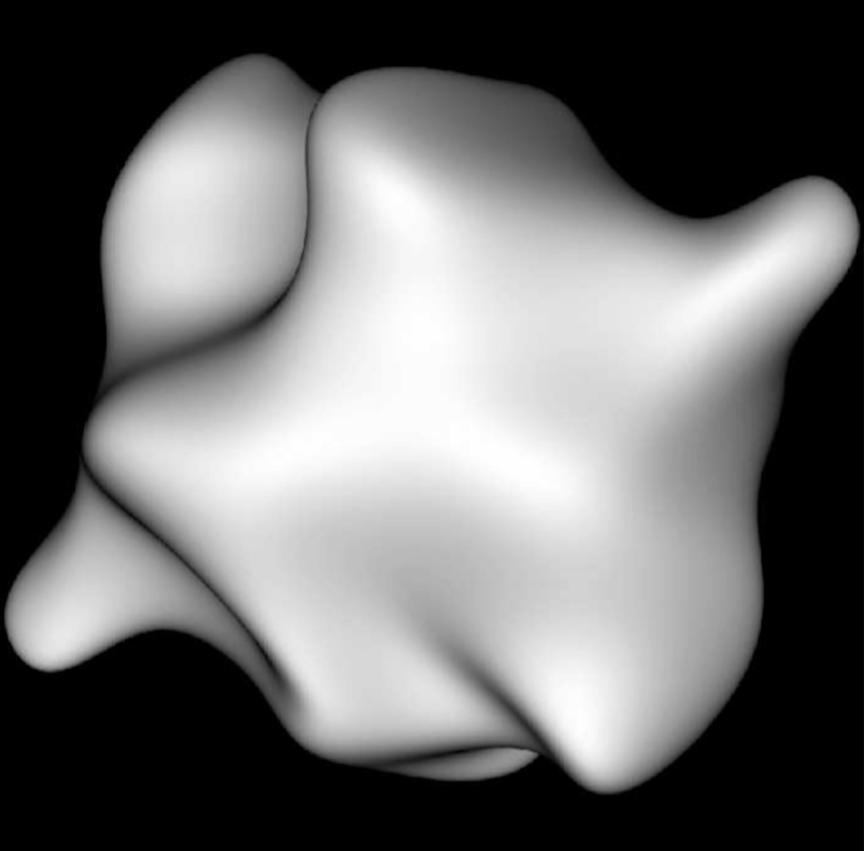}
b) \includegraphics[width = 0.25 \linewidth]{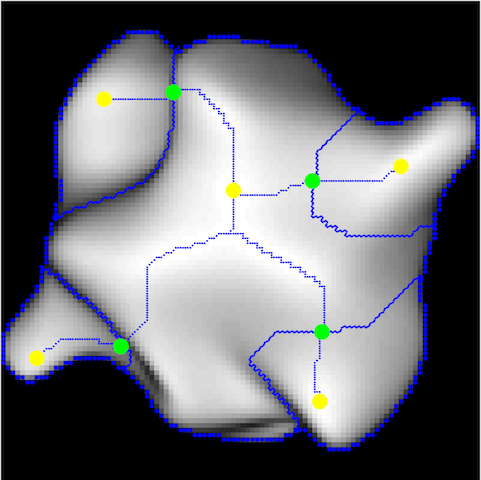}
c) \includegraphics[width = 0.32 \linewidth]{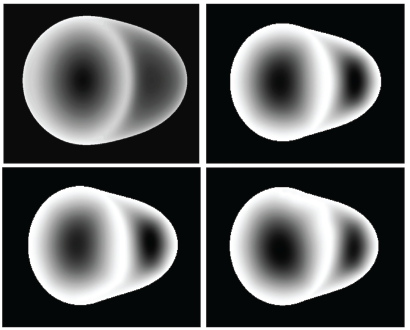}
\caption{{\rm (a), (b)} This figure illustrates how a surface is segmented by the full MS complex for the shading into salient parts. Notice the four major lobes, pointing outward like
the ends of an ``X,'' plus some interior parts. Crisp maxima in intensity signal the four dominant lobes. The middle maximum signals an interior part.  {\rm (c)}  True slant function (upper left).  Remaining three figures: Sample slant reconstructions from images in row $1$ of Figure {\rm \ref{fig:furrow_MS}}.  We used a linear inpainting algorithm {\rm \cite{DErrico}} and knowledge of the position of red critical contour corresponding to a $1$-cell in each of the images in Figure {\rm \ref{fig:furrow_MS}}.  We also used slant information at the occluding boundary, where the normal must be perpendicular to the view direction.  Note the strong similarity between the slant reconstructions even though the original images are pointwise very different.}
\label{fig:recon_1_cells}
\end{center}
\end{figure}

\section{Conclusion}

We seek a biologically plausible approach to 3D inferences in which an imaged surface is represented via a set of (isophote) contours or equivalent flows.  This allows us to separate those portions of the flow that are stable from those that wander, with respect to lighting and the image formation process. We believe this will result in a more robust, nearly invariant approach. To achieve this,  we have defined critical contours, computable from the image, with two important characteristics.  First, we showed that they are stable over changes in the rendering function.  Second, they relate to the MS complex of the surface slant.  This allows us to interpret critical contours as boundaries of surface features. Thus, the $K$-critical contours are part of a meaningful segmentation of the surface shared by almost all our admissible renderings.  (As $K \rightarrow \infty$, ``almost'' becomes ``all.'')  It is these stable contours with which we hope to transition from a local (individual gradients) representation to a more global (unions of bumps) representation.

Further, using the MS complex reveals relationships between shading inferences and shape-from-sketching, under the same model.  Certain (e.g., isotropic) textures may also allow for a similar analysis, when based on estimated foreshortening rather than intensity.  In addition, the invariance of the MS complex to monotonic transformations relates to psychophysical observations seen in \cite{doi:10.1167/13.5.10, vergne12}.  Modeling with the MS complex allows us to assign meaning to individual contours as, e.g., the boundary of a ``bump.''  By seeking this much weaker surface structure than, e.g., a 3D mesh, we hope to avoid most of the ill-posedness inherent in the 3D reconstruction problem.

We are focusing on two future directions.  Critical contours rarely completely segment the image, as the MS complex may also have unstable 1-cells.  Thus, we are pursuing methods  to find the ``nearest segmentation'' to a set of critical contours to complete the complex.  Second, we are analyzing the qualitative conclusions that can be drawn from a full segmentation.  A constraint labeling problem arises, namely, which contours  bound locally convex parts and which bound locally concave parts.

To summarize, we are arguing for the use of critical contours in 3D shape reconstruction from shading and contours.  These critical contours are part of the MS complex of the image and give a shape description that is stable, qualitative, and meaningful.  Therefore, reconstruction algorithms explicitly using these image features should be more stable while also explicitly capturing important surface features.  We believe studying these topological properties will aid our understanding of how the human visual system is able to see a veridical 3D shape under complex and noisy renderings.

\section{Appendix}

\subsection{Comparison of critical contours and suggestive contours}
Suggestive contours are contours (drawn from the surface mesh) that illustrate shape \cite{DeCarlo03}.   Their generating equation (notation from \cite{DeCarlo03}) is $D_{\bm{w}} (\bm{N} \cdot \bm{v}) = 0$, where $\bm{v}$ is the view direction and $\bm{w}$ is the view direction projected onto the tangent plane.  They are the set of minima of $N \cdot \bm{v}$ in direction $\bm{w}$.  The slant function $\sigma(x, y)$ is simply the angle between $N$ and $\bm{v}$ and so is an $\arccos$ transformation of $N \cdot \bm{v}$.  We note that extremal curves are invariant under strictly monotonic transformations via the chain rule, so the suggestive contours can also be seen as maxima of the slant function $\sigma$ in the direction $\bm{w}$.

Under orthographic projection, it is a simple calculation to see that $\bm{w}$ is proportional to the surface gradient $\nabla S$ (or tilt direction) projected onto the image plane.  Thus, under orthographic projection, we can rewrite the generating equation for suggestive contours as the set of points satisfying $\nabla S \cdot \nabla \sigma = 0$.


We compare to a critical contour; these are the 1-cells of $\sigma$ by Theorem \ref{main_theorem}, yet computable from the image.   We see that suggestive contours depend on whether $\nabla \sigma$ points in the gradient direction, whereas critical contours depend on the global properties of the $\nabla \sigma$ field.


%

\subsection{Proof of Lemma \ref{lemma:shading_derivs}}
\begin{customthm}{3}
Let $\alpha(t), I_\alpha, \bm{u}, \bm{w}$ be defined as in section  {\rm \ref{sec:contour_is_shading}}.  The sequence of shaded images $\{ I^\sigma \}$ converge pointwise to the original line drawing $I_\alpha$ and have the following properties on the derivatives as $\sigma \rightarrow 0$:
\begin{enumerate}\leftskip38pt
\item[{\rm (1)}] $I^\sigma_{ww} (\alpha(t)) \rightarrow -\infty$ for every $t \in [0, 1]$.
\item[{\rm (2)}] $I^\sigma_{uu} (\alpha(t)) \rightarrow \infty$ for $t = \{0, 1\}$.
\item[{\rm (3)}] There exists a constant $M$ such that  $| I^\sigma_{w} |, |I^\sigma_{u}|, |I^\sigma_{uw}| < M$ for every $t \in [0, 1]$.
\end{enumerate}
\end{customthm}


\begin{proof}
Start from \eqref{eqn:shading_limit}:
\begin{align}
I^\sigma (\xi, \eta) & =  \frac{1}{\sqrt{2 \pi} \sigma} \int I_\alpha (p + x \bm{u}) e^{-\frac{\eta^2}{2 \sigma^2}}    e^{-\frac{(x - \xi)^2}{2 \sigma^2}} \, dx \\
& = (I_\alpha \ast G_\sigma) e^{-\frac{\eta^2}{2 \sigma^2}}, \
\end{align}
which defines a sequence of shaded images so that $\lim_{\sigma \rightarrow 0} I^{\sigma} = I_\alpha$.  We now compute image derivatives, up to second order, along our contour for these shaded images, $I^{\sigma}$.  Taking the limit as $\sigma \rightarrow 0$, we will inherit derivatives in the limit for $I_\alpha$.


\vskip-\lastskip{\pagebreak}

We differentiate $I^\sigma (\xi, \eta)$ in the tangent direction:
\begin{align}
\frac{\partial}{\partial \xi} I^\sigma (\xi, \eta) & = \frac{\partial}{\partial \xi}  \left( (I_\alpha \ast G_\sigma) e^{-\frac{\eta^2}{2 \sigma^2}} \right)\\
& = \left( \left( \frac{\partial}{\partial \xi}  I_\alpha \right) \ast G_\sigma \right) e^{-\frac{\eta^2}{2 \sigma^2}}. \
\end{align}

For each point $p$ on $\alpha(t), t \in (0, 1)$, the limit as $\sigma \rightarrow 0$ is $\frac{\partial}{\partial \xi} I_\alpha \rvert_p$ as expected.  (That is, the image derivative along the contour $\alpha$ is the limit of the derivatives along the shading approximations to the contour.)

We repeat the same process for the remainder of the derivatives up to second order and get
\begin{align}
\lim_{\sigma \rightarrow 0} \frac{\partial}{\partial \eta} I^\sigma (\xi, \eta) \biggr \rvert_p & = 0,\\
\lim_{\sigma \rightarrow 0} \frac{\partial^2}{\partial \xi^2} I^\sigma (\xi, \eta) \biggr \rvert_p & = I_\alpha''(t), \\
\lim_{\sigma \rightarrow 0} \frac{\partial^2}{\partial \eta^2} I^\sigma (\xi, \eta) \biggr \rvert_p\  & = -\infty,\\
\lim_{\sigma \rightarrow 0}  \frac{\partial}{\partial \xi} \frac{\partial}{\partial \eta} I^\sigma (\xi, \eta) \biggr \rvert_p\ & = 0.\
\end{align}

We also would like the image derivatives at the endpoints of $\alpha(t)$.  To calculate these approximations, we apply a Heaviside step function $H_0 (x)$ to the endpoints of the curve $\alpha(t)$.  To make the integral feasible, we Taylor approximate the intensity on the contour $I_\alpha (p + x \bm{u})$ up to second order:
\begin{align}
 I_\alpha (p + x \bm{u}) & = c_0 + c_1 x + c_2 x^2 \
\end{align}
for some constants $\{c_0, c_1, c_2\}$.  This ``approximation'' becomes exact as $\sigma \rightarrow 0$.  If we are at an endpoint $\alpha(0)$ and we move in the positive tangent direction, the image intensity is defined by this Taylor expansion.  If we move in the negative tangent direction, the image intensity is zero.  For example,
\begin{align}
\frac{\partial}{\partial \xi} I^\sigma (\xi, \eta) \biggr \rvert_{\alpha(0)}  & = \int \frac{\partial}{\partial \xi}  \left( I_\alpha (p + x \bm{u}) e^{-\frac{\eta^2}{2 \sigma^2}}    e^{-\frac{(x - \xi)^2}{2 \sigma^2}} H_0 (x) \right) \, dx \\
& = \int \frac{\partial}{\partial \xi}  \left( (c_0 + c_1 x + c_2 x^2 ) e^{-\frac{\eta^2}{2 \sigma^2}}    e^{-\frac{(x - \xi)^2}{2 \sigma^2}} H_0 (x) \right) \, dx \\
& =  \frac{c_1}{2} + \frac{c_0 + 2 c_2 \sigma^2}{\sqrt{2 \pi} \sigma}. \
\end{align}

As we require the contour intensity to be $0$ at the endpoint, we can set $c_0 = 0$ and calculate the limit of $\frac{\partial}{\partial \xi} I^\sigma (\xi, \eta) \rvert_{\alpha(0)}$ as $\epsilon \rightarrow 0$ to get $\frac{c_1}{2}$.

We can also calculate the other image derivatives at the endpoint $\alpha(0)$.  (Note that the other endpoint, $\alpha(1)$, is just the mirror version; we use $-H_0(x)$ instead of $H_0 (x)$.)\pagebreak\vspace*{-20pt}
\begin{align}
\lim_{\sigma \rightarrow 0} \frac{\partial}{\partial \eta} I^\sigma (\xi, \eta) \biggr \rvert_{\alpha(0)} & = 0,\\
\lim_{\sigma \rightarrow 0} \frac{\partial^2}{\partial \xi^2} I^\sigma (\xi, \eta) \biggr \rvert_{\alpha(0)} & = \infty, \\
\lim_{\sigma \rightarrow 0} \frac{\partial^2}{\partial \eta^2} I^\sigma (\xi, \eta) \biggr \rvert_{\alpha(0)} & = - \infty,\\
\lim_{\sigma \rightarrow 0}  \frac{\partial}{\partial \xi} \frac{\partial}{\partial \eta} I^\sigma (\xi, \eta) \biggr \rvert_{\alpha(0)}\ & = 0.\
\end{align}
These calculations are consolidated into Lemma \ref{lemma:shading_derivs}.
\end{proof}



\subsection{Proof of Lemmas \ref{lemma:crit_points} and \ref{lemma:beta_curves}}
We first prove the following lemma, Lemma \ref{lemma1}, that allows us to bound terms in a Taylor expansion of $\bm{DN}$ and $\bm{D^2 N}$ from image derivatives at a point.

\begin{lemma}
Assume the generic and rendering function assumptions in section {\rm \ref{sec:image_formation}} hold.  Let $p$ be a point on a $K$-critical contour $\alpha(t)$.
If $\, | D I_{p}| < M$, then $|| \bm{DN}_{N(p)} (\cdot) ||$ is bounded.  Similarly, $\, | D I_{p} | < M$  and $|I_{uw} (p)| < M$ imply $|| \bm{D^2 N}_{N(p)} (u, w) ||$ is bounded.
\label{lemma1}
\end{lemma}

\begin{proof}
~

\begin{align}
|DI_{p} ( \cdot)  | & < M,\\
 |DF^T \bm{DN}_{N(p)} (\cdot )| & < M. \
\end{align}

By genericity property 2, we must have $|| \bm{DN}_{N(p)} ( \cdot) ||$ bounded in Frobenius norm ${\rm for\ all}\break p \in \alpha(t)$.  We see here that this prevents an infinitesimal change in the rendering function (that is, $DF$) resulting in an unbounded change in the image gradient $\nabla I$.

We repeat the same argument for $| I_{uw} | < M$, taking one further derivative and leaving off the $p$ subscript for clarity:
\begin{align}
| I_{uw} | & < M, \\
| \bm{D^2 F} ( \bm{DN}(u), \bm{DN}({w})) + {DF}^T \bm{D^2 N}({u}, {w}) | & < M,\\
| {DF}^T \bm{D^2 N}({u}, {w})| & <  M + | \bm{D^2 F} ( \bm{DN}({u}), \bm{DN}({w}))|. \
\end{align}

On the left-hand side, $| {DF}^T \bm{D^2 N}({u}, {w})|$ is bounded as each of $|| \bm{D^2 F} ||,  | \bm{DN}({u}) |,  | \bm{DN}({w}) |$ are bounded.  Applying the second generic property, we see that generically $| \bm{D^2 N} ({u}, {w}) |$ is also bounded.
\end{proof}

\subsubsection{Proof of Lemma  \ref{lemma:crit_points}}
\begin{customthm}{6}
Let $a \in \Omega_\alpha \subset \mathbb{R}^2$ be an endpoint of the $K$-critical contour with the conditions from Definition {\rm \ref{defn_cc}}.  Given a new rendering function $\tilde{F}$, resulting image $\tilde{I}$, and any $\delta > 0$, there exists a $K_0$ such that if $K > K_0$, the following holds:  $ \exists \, \,\tilde{a} \,\in \Omega_\alpha,$ such that $|| a - \tilde{a} || < \delta$ and $\tilde{a}$ is a critical point of $\tilde{I}$.
\end{customthm}

\begin{proof}
We will consider the image of the normal field on the Gauss sphere in the neighborhood of $N(a)$.  The main idea is that if a differentiable function has a large enough gradient at a point $a$, then it has a zero inside a neighborhood of $a$.  We take two derivatives of the equation $I(x, y) = F({N}(x, y))$ to get the following equation of tensors:
\begin{align}
\label{eqn:Iww}
\bm{D^2 I}_{(x, y)} ({u_1}, {u_2}) & = \bm{D^2 F}( \bm{DN}_{(x, y)} ({u_1}), \bm{DN}_{(x, y)} ({u_2})) + {DF}^T \bm{D^2 N}_{(x, y)} ({u_1}, {u_2}). \
\end{align}

To calculate $I_{ww} (a) $, we replace  ${u_1}, {u_2}$ both with ${w}$ and let $(x, y) = a$, where $a$ is any point on $\alpha(t)$.  By the concave rendering function assumption,
\begin{align}
\label{eqn:concave}
0 > \bm{D^2 F}( \bm{DN}_{a} ({w}), \bm{DN}_{a} ({w})). \
\end{align}

If $I_{ww} > K$, then
\begin{align}
K < {DF}^T \bm{D^2 N}_{a} ({w}, {w}). \
\end{align}

As $|| {\nabla F}|| < C_1$ by the rendering function assumptions in Definition \ref{rendering_funct_assum},
\begin{align}
\label{eqn:bound1}
\frac{K}{C_1} < || \bm{D^2 N}_{a} ({w}, {w}) ||.
\end{align}

Similarly, for $I_{uu}$, we get
\begin{align}
\label{eqn:bound2}
\frac{K}{C_1} < || \bm{D^2 N}_{a} ({u}, {u}) ||.
\end{align}

Recall that ${D \tilde{F}}$ is the differential of the second rendering function.  We expand the operator ${D \tilde{F}}^T \bm{DN}_{a}( \cdot)$ with a first order multivariate Taylor expansion of  $\bm{DN}_{a}( \cdot)$ around $a$.   For example, the derivative of $\bm{DN}_{a}( \cdot)$ in the $u$ direction is $\bm{D^2 N}_a ({u}, \cdot)$.
\begin{align}
\label{eqn:T_expand}
{D \tilde{F}}^T \bm{DN}_{a + q {u} + r {w}}  ( \cdot ) = {D \tilde{F}}^T \bm{DN}_a ( \cdot) +  q \, {D \tilde{F}}^T \bm{D^2 N}_a ({u}, \cdot) + r \, {D \tilde{F}}^T \bm{D^2 N}_a ({w}, \cdot) \
\end{align}

From \eqref{eqn:bound1} and \eqref{eqn:bound2}, we know that $|| \bm{D^2 N}_{N_a} ({u}, {u}) ||$ and $|| \bm{D^2 N}_{N_a} ({w}, {w}) ||$ are sufficiently large; we want to find a $(q_0, r_0)$ such that ${D \tilde{F}}^T \bm{DN}_{N_a + q_0 {u} + r_0 {w}}  ( \cdot )$ is precisely 0.

We use the first generic property to assume that the span of three vectors
$\{\bm{D^2 N}_{a} ({u}, {u}),\break  \bm{D^2 N}_{a} ({u}, {w}),  \bm{D^2 N}_{a} ({w}, {w}) \}$ contains at least two linearly independent ones.  This implies that ${D \tilde{F}}^T \bm{D^2 N}_a ({w}, \cdot)$ and ${D \tilde{F}}^T \bm{D^2 N}_a ({u}, \cdot)$ are not parallel vectors and thus they span a plane $P$. Generically, $P$ contains an intersection point with the unknown vector $-{D \tilde{F}}^T \bm{DN}_a ( \cdot)$.  That intersection point defines a $\{q_0, r_0\}$ satisfying ${D \tilde{F}}^T \bm{DN}_{a + q_0 {u} + r_0 {w}}  ( \cdot ) = 0$.

Define $\tilde{a} = a + q_0 {u} + r_0 {w}$.  It remains to show $|| a - \tilde{a}|| < \epsilon_{(C_1, C_2)}$.  Recall that ${D \tilde{F}}^T \bm{DN}_a ( \cdot )$ is a bounded vector by Lemma \ref{lemma1}. From \eqref{eqn:T_expand},
\begin{align}
-{D \tilde{F}}^T \bm{DN}_a (\cdot) & = q_0 \, {D \tilde{F}}^T \bm{D^2 N}_a ({u}, \cdot) + r_0 \, {D \tilde{F}}^T \bm{D^2 N}_a ({w}, \cdot). \
\label{eqn:matrix_equation}
\end{align}

The above matrix equation \eqref{eqn:matrix_equation} represents two equations:
\begin{align}
-{D \tilde{F}}^T \bm{DN}_a ({u}) & - q_0 \, {D \tilde{F}}^T \bm{D^2 N}_a ({u}, {u}) = r_0 \, {D \tilde{F}}^T \bm{D^2 N}_a ({w}, {u}),  \label{1st_vectorial_eqn}\\
-{D \tilde{F}}^T \bm{DN}_a ({w}) & - r_0 \, {D \tilde{F}}^T \bm{D^2 N}_a ({w}, {u}) = q_0 \, {D \tilde{F}}^T \bm{D^2 N}_a ({w}, {w}). \label{2nd_vectorial_eqn}\
\end{align}

We note that ${D \tilde{F}}^T \bm{DN}_a ({u})$  and $r_0 \, {D \tilde{F}}^T \bm{D^2 N}_a ({w}, {u})$ are bounded by Lemma \ref{lemma1}. We can then bound the first and third terms of \eqref{1st_vectorial_eqn} with some $\Gamma \in \mathbb{R}$:
\begin{align}
|{D \tilde{F}}^T \bm{DN}_a ({u})| +   |r_0 \, {D \tilde{F}}^T \bm{D^2 N}_a ({w}, {u})| & < \Gamma, \\
| q_0 {D \tilde{F}}^T \bm{D^2 N}_a ({u}, {u}) | & < \Gamma,\\
|q_0|\, || {D \tilde{F}} ||\, ||\bm{D^2 N}_a ({u}, {u}) ||\,  | \cos(\theta) |& < \Gamma, \\
|q_0| & < \frac{\Gamma}{ | \cos(\theta) | \, ||\bm{D^2 N}_a ({u}, {u}) || | {D \tilde{F}} |}, \label{req_generic_1} \\
|q_0| & < \frac{\Gamma C_1}{\epsilon^2 \, K},  \label{req_generic_2} \
\end{align}
where $\theta$ is the angle between the vectors ${D \tilde{F}}$ and $\bm{D^2 N}_a ({u}, {u})$.  By generic assumptions \eqref{gen_assumptions}, $| {D \tilde{F}} |, | \cos(\theta)| $ are bounded below by $\epsilon$.  To go from \eqref{req_generic_1} to  \eqref{req_generic_2}, we also substituted in from \eqref{eqn:bound2}.  Similarly, from   \eqref{2nd_vectorial_eqn}, we get $|r_0| < \frac{\Gamma C_1}{\epsilon^2 \, K}$.

Now, $r_0$ and $q_0$ are the displacements from point $a$ to point $\tilde{a}$:
\begin{align}
|| a - \tilde{a} || & = \sqrt{r_0^2 + q_0^2} \\
& < \sqrt{2} \frac{\Gamma C_1}{\epsilon^2 \,K}. \
\end{align}
Define $\delta (K) = \sqrt{2} \frac{\Gamma C_1}{\epsilon^2 \,K} $ to complete the proof.
\end{proof}

\subsubsection{Proof of Lemma \ref{lemma:beta_curves}}
\begin{customthm}{7}
Let $\alpha(t)$ be a $K$-critical contour with the conditions from Definition {\rm \ref{defn_cc}}.   Recall that $w(t)$ is the transversal direction in $\mathbb{R}^2$ to $\alpha(t)$ at each $t$.  Given a new rendering function $\tilde{F}$, resulting image $\tilde{I}$, and $\delta > 0$, there exists a $K_0$ such that if $K > K_0$, the following holds: Define two curves $\beta_1(t) =  \alpha(t) + \delta w$ and  $\beta_2(t) = \alpha(t) -\delta w $.   On $\beta_1(t)$, $D \tilde{I}_{\beta_1(t)} (w(t)) > 0$ and on $\beta_2(t)$, $D \tilde{I}_{\beta_1(t)} (w(t)) < 0$.

\end{customthm}

\begin{proof}
~
\begin{align}
K & < I_{ww} \\
& < {DF}^T \bm{D^2 N}({w}, {w})
\end{align}
by applying \eqref{eqn:Iww} and \eqref{eqn:concave}.  By the rendering function assumptions $|| {DF} ||$ is bounded by $C_1$:
\begin{align}
\frac{K}{C_1} & < || \bm{D^2 N}({w}, {w}) ||.
\end{align}

By the generic assumptions described in \eqref{gen_assumptions},
\begin{align}
|{D \tilde{F}}^T \bm{D^2 N}({w}, {w}) |  & = | {D \tilde{F}} |\, || \bm{D^2 N}({w}, {w}) || \, \cos(\theta) \\
& > \frac{K \epsilon^2}{C_1} \label{eqn:Sww_bound}
\end{align}

Expand the scalar function ${D \tilde{F}}^T \bm{DN}({w})$ to first order in the ${w}$ direction around any point on the contour $\alpha(t)$:
\begin{align}
D \tilde{I}_{\alpha(t) + s {w}} (w) & = {D \tilde{F}}^T \bm{DN}_{\alpha(t) + s {w}} ({w})  \\
& \approx {D \tilde{F}}^T \bm{DN}_{\alpha(t)} ({w}) + s \, {D \tilde{F}}^T \bm{D^2 N}_{\alpha(t)}({w}, {w}).
\label{vDNw}
\end{align}

Fix a $\delta$; it suffices to show that this quantity is positive for $s = \delta$ and negative for $s  = - \delta$.  If so, then we have shown, e.g., $D \tilde{I}_{\beta_1(t)} (w) > 0$ for ${\beta_1(t)} = \alpha(t) + \delta {w}(t)$.

Now note that the first term on the right-hand side of \eqref{vDNw} is upper bounded for all $t$ on the critical contour $\alpha(t)$ by Lemma \ref{lemma1}.  And note that the second term is lower bounded proportional to $K$ from \eqref{eqn:Sww_bound}.  Thus, for the given $\delta$ and for all $t$, we can choose a $K$ such that
\begin{align}
|{D \tilde{F}}^T \bm{DN}_{\alpha(t)} ({w})| < | \delta \, {D \tilde{F}}^T \bm{D^2 N}_{\alpha(t)}({w}, {w}) |.
\end{align}
We conclude that in \eqref{vDNw}, the first term is dominated by the second term, whose sign is dependent on $s$.  Thus, we get the necessary gradient conditions on  $D \tilde{I}$.
\end{proof}

\subsection*{Acknowledgments}
We thank Steve Cholewiak, Roland Fleming, and Daniel Holtmann-Rice for useful discussions.

\end{document}